\documentclass[journal,onecolumn,12pt]{IEEEtran}
\date{}

\usepackage{amsfonts}
\usepackage{amsmath}
\usepackage{amssymb}
\usepackage{amsthm}
\usepackage{array}
\usepackage{caption}
\usepackage{cite}
\usepackage{citesort}
\usepackage{color}
\usepackage{epsfig}
\usepackage{enumitem} 
\usepackage{flexisym}
\usepackage{graphicx}
\usepackage{hyphenat}
\usepackage{ifpdf}
\usepackage{latexsym}
\usepackage{mathtools}
\usepackage{multirow}
\usepackage{psfrag}
\usepackage{setspace}
\usepackage{subfigure}
\usepackage{tikz}
\usepackage{times}
\usepackage{tkz-berge}
\usepackage{xparse}

\newtheorem{theorem}{Theorem}
\newtheorem{definition}{Definition}
\newtheorem{proposition}{Proposition}
\newtheorem{lemma}{Lemma}
\newtheorem{corollary}{Corollary}

\newtheorem{remark}{Remark}

\title{Rank Determination for Low-Rank Data Completion}

\author{Morteza Ashraphijuo, Xiaodong Wang and Vaneet Aggarwal\thanks{Morteza Ashraphijuo and Xiaodong Wang are with the Department of Electrical Engineering, Columbia University, NY, email: \{ashraphijuo,wangx\}@ee.columbia.edu. Vaneet Aggarwal is with the School of Industrial Engineering, Purdue University, West Lafayette, IN, email: vaneet@purdue.edu. }}
 
\begin{document}
\maketitle


\begin{abstract}
Recently, fundamental conditions on the sampling patterns have been obtained for finite completability of low-rank matrices or tensors given the corresponding ranks. In this paper, we consider the scenario where the rank is not given and we aim to approximate the unknown rank based on the location of sampled entries and some given completion. We consider a number of data models, including single-view matrix, multi-view matrix, CP tensor, tensor-train tensor and Tucker tensor. For each of these data models, we provide an upper bound on the rank when an arbitrary low-rank completion is given. We characterize these bounds both deterministically, i.e., with probability one given that the sampling pattern satisfies certain combinatorial properties, and probabilistically, i.e., with high probability given that the sampling probability is above some threshold. Moreover, for both single-view matrix and CP tensor, we are able to show that the obtained upper bound is exactly equal to the unknown rank if the lowest-rank completion is given. Furthermore, we provide numerical experiments for the case of single-view matrix, where we use nuclear norm minimization to find a low-rank completion of the sampled data and we observe that in most of the cases the proposed upper bound on the rank is equal to the true rank.

\

\begin{IEEEkeywords}
Low-rank data completion, rank estimation, tensor, matrix, manifold, Tucker rank, tensor-train rank, CP rank, multi-view matrix.
\end{IEEEkeywords}

\end{abstract}

\newpage

\section{Introduction}

The low-rank data completion problem is concerned with completing a matrix or tensor given a subset of its entries and some rank constraints. Various applications can be found in many fields including image and signal processing \cite{phase,Image}, data mining \cite{data}, network coding \cite{network}, compressed sensing \cite{lim,sid,gandy}, reconstructing the visual data \cite{visual}, etc. There is an extensive literature on developing various optimization methods to treat this problem including minimizing a convex relaxation of rank \cite{candes,candes2,cai,gandy,ashraphijuo2016c}, non-convex approaches \cite{recllel}, and alternating minimization  \cite{jain2013low,ge2016matrix}, etc. More recently, fundamental conditions on the sampling pattern that lead to different numbers of completion (unique, finite, or infinite) given the rank constraints have been investigated in \cite{charact,ashcon2,ashraphijuo4,ashcon1,ashraphijuo3}.

However, in many practical low-rank data completion problems, the rank may not be known {\it a priori}. In this paper, we investigate this problem and we aim to approximate the rank based on the given entries, where it is assumed that the original data is generically chosen from the manifold corresponding to the unknown rank. The only existing work that treats this problem for a single-view matrix data based on the sampling pattern is \cite{converse}, which requires some strong assumptions including the existence of a completion whose rank $r$ is a lower bound on the unknown true rank $r^*$, i.e., $r^* \geq r$. We start by investigating the single-view matrix  to provide a new analysis that does not require such assumption and also we can  extend our approach to treat the CP rank tensor model. Moreover, we further generalize our approach to treat vector rank data models including the multi-view matrix, the Tucker rank tensor and the tensor-train (TT) rank tensor. For each of these data models, we obtain the upper bound on the scalar rank or component-wise upper bound on the unknown vector rank, deterministically based on the sampling pattern and the rank of a given completion. We also obtain such bound that holds with high probability based on the sampling probability. Moreover, for the single-view matrix, we provide some numerical results to show how tight our probabilistic bounds on the rank are (in terms of the sampling probability). In particular, we used nuclear norm minimization to find a completion and demonstrate our proposed method in obtaining a tight bound on the unknown rank.

We take advantage of the geometric analysis on the manifold of the corresponding data which leads to the fundamental conditions on the sampling pattern (independent of the value of entries) \cite{charact,ashraphijuo2,ashraphijuo4,ashraphijuo,ashraphijuo3} such that given an arbitrary low-rank completion we can provide a tight upper bound on the rank. To illustrate how such approximation is even possible consider the following example. Assume that an $n_1 \times n_2$ rank-$2$ matrix is chosen generically from the corresponding manifold. Hence, any $2 \times 2$ submatrix of this matrix is full-rank with probability one (due to the genericity assumption). Moreover, note that any $3 \times 3$ submatrix of this matrix is not full-rank. As a result, by observing the sampled entries we can find some bounds on the rank. Using the analysis in \cite{charact,ashraphijuo2,ashraphijuo4,ashraphijuo,ashraphijuo3} on finite completablity of the sampled data (finite number of completions) for different data models, we characterize both deterministic and probablistic bounds on the unknown rank.

The remained of the paper is organized as follows. In Section \ref{modelno}, we introduce the data models and problem statement. In Sections \ref{semnw} and \ref{rankvecmo} we characterize our determintic and probablistic bounds for scalar-rank cases (single-view matrix and CP tensor)  and vector-rank cases (multi-view matrix, Tucker tensor and TT tensor), respectively. Finally, Section \ref{conclusection} concludes the paper.



\section{Data Models and Problem Statement}\label{modelno}

\subsection{Matrix Models}\label{modelmat}

\subsubsection{Single-View Matrix}

Assume that the sampled matrix $\mathbf{U}$ is chosen generically from the manifold of the $n_1 \times n_2$ matrices of rank $r^*$, where $r^*$ is unknown. The matrix $\mathbf{V} \in \mathbb{R}^{n_1 \times r^*}$ is called a basis for $\mathbf{U}$ if each column of $\mathbf{U}$ can be written as a linear combination of the columns of $\mathbf{V}$. Denote $\mathbf{\Omega}$ as the binary sampling pattern matrix that is of the same size as $\mathbf{U}$ and $\mathbf{\Omega}(\vec{x})=1$ if $\mathbf{U}(\vec{x})$ is observed and  $\mathbf{\Omega}(\vec{x})=0$ otherwise, where $\vec{x}=(x_1,x_2)$ represents the entry corresponding to row number $x_1$ and column number $x_2$. Moreover, define $\mathbf{U}_{\mathbf{\Omega}}$ as the matrix obtained from sampling $\mathbf{U}$ according to $\mathbf{\Omega}$, i.e., 
\begin{eqnarray}
\mathbf{U}_{\mathbf{\Omega}} (\vec{x}) = \left\{
	\begin{array}{ll}
		\mathbf{U} (\vec{x})  & \mbox{if } \ \mathbf{\Omega}(\vec{x}) \ = 1, \\
		0  & \mbox{if } \ \mathbf{\Omega}(\vec{x}) \ = 0.
	\end{array}
\right.
\end{eqnarray}

\subsubsection{Multi-View Matrix}

The matrix $\mathbf{U} \in \mathbb{R}^{n \times (n_1+n_2)}$ is sampled. Denote a partition of $\mathbf{U}$ as $\mathbf{U}= [\mathbf{U}_1|\mathbf{U}_2]$ where $\mathbf{U}_1 \in \mathbb{R}^{n \times n_1}$ and $\mathbf{U}_2 \in \mathbb{R}^{n \times n_2}$ represent the first and second views of data, respectively. The sampling pattern is defined as $\mathbf{\Omega} = [\mathbf{\Omega}_1|\mathbf{\Omega}_2]$, where $\mathbf{\Omega}_1$ and $\mathbf{\Omega}_2$ represent the sampling patterns corresponding to the first and second views of data, respectively. Assume that $\text{rank}(\mathbf{U}_1)=r_1^*$, $\text{rank}(\mathbf{U}_2)=r_2^*$ and $\text{rank}(\mathbf{U})=r^*$, and also $\mathbf{U}$ is chosen generically from the manifold structure with above parameters. Denote $\underline{r}^*=(r_1^*,r_2^*,r^*)$ which is assumed unknown. 

\subsection{Tensor Models}\label{modelten}

Assume that a $d$-way tensor $\mathcal{U} \in \mathbb{R}^{n_1  \times \cdots \times n_d}$ is sampled. For the sake of simplicity in notation, define  $N_{i} \triangleq \left( \Pi_{j=1}^{i} \  n_j \right)$, $\bar N_{i} \triangleq \left( \Pi_{j=i+1}^{d} \  n_j \right)$ and $N_{-i} \triangleq \frac{N_d}{n_i}$. Denote $\Omega$ as the binary sampling pattern tensor that is of the same size as $\mathcal{U}$ and $\Omega(\vec{x})=1$ if $\mathcal{U}(\vec{x})$ is observed and  $\Omega(\vec{x})=0$ otherwise, where $\mathcal{U}(\vec{x})$ represents an entry of tensor $\mathcal{U}$  with coordinate $\vec{x}=(x_1,\dots,x_d)$. Moreover, define $\mathcal{U}_{{\Omega}}$ as the tensor obtained from sampling $\mathcal{U}$ according to ${\Omega}$, i.e.,
\begin{eqnarray}
\mathcal{U}_{{\Omega}} (\vec{x}) = \left\{
	\begin{array}{ll}
		\mathcal{U} (\vec{x})  & \mbox{if } \ {\Omega}(\vec{x}) \ = 1, \\
		0  & \mbox{if } \ {\Omega}(\vec{x}) \ = 0.
	\end{array}
\right.
\end{eqnarray}
For each subtensor $\mathcal{U}^{\prime} $ of the tensor $\mathcal{U}$, define $N_{\Omega}({\mathcal{U}^{\prime}})$ as the number of observed entries in $\mathcal{U}^{\prime}$ according to the sampling pattern $\Omega$.

Define the matrix $\mathbf{\widetilde U}_{(i)} \in \mathbb{R}^{N_i\times  \bar N_{i}}$ as the $i$-th {\em unfolding} of the tensor $\mathcal{U}$, such that $\mathcal{U}(\vec{x}) = \\ \mathbf{\widetilde U}_{(i)}({\widetilde  M}_{i} (x_1,\dots,x_i),{\widetilde{{M}} }_{-i} (x_{i+1},\ldots,x_d))$, where ${\widetilde M}_{i}: (x_1,\dots,x_i) \rightarrow  \{1,2,\dots, N_i\}$ and ${\widetilde{{M}}}_{-i}: (x_{i+1},\ldots,x_d)  \rightarrow  \{1,2,\dots, \bar N_{i} \}$ are two bijective mappings.

Let $\mathbf{U}_{(i)} \in \mathbb{R}^{n_i \times N_{-i}}$ be the $i$-th {\em matricization} of the tensor $\mathcal{U}$, such that $\mathcal{U}(\vec{x}) = \\ {\mathbf{U}}_{(i)}(x_i,{M}_{i} (x_1,\ldots,x_{i-1},x_{i+1},\ldots,x_d))$, where ${M}_{i}: (x_1,\ldots,x_{i-1},x_{i+1},\ldots,x_d) \rightarrow \{1,2,\dots, N_{-i} \}$  is a bijective mapping. Observe that for any arbitrary tensor $\mathcal{A}$, the first matricization and the first unfolding are the same, i.e., $\mathbf{ A}_{(1)} = \mathbf{\widetilde A}_{(1)}$.

In what follows, we introduce three different tensor ranks, i.e., the CP rank, Tucker rank and TT rank.

\subsubsection{CP Decomposition}\label{CPmodel}

The CP rank  of a tensor $\mathcal{U}$, $\text{rank}_{\text{CP}} (\mathcal{U})= r$, is defined as the minimum number $r$ such that there exist $\mathbf{a}_{i}^{l} \in \mathbb{R}^{n_i}$ for $1 \leq i \leq d$ and $1 \leq l \leq r$, such that
\begin{eqnarray}\label{CPdecom}
\mathcal{U} = \sum_{l=1}^{r}  \mathbf{a}_{1}^{l} \otimes \mathbf{a}_{2}^{l} \otimes \dots \otimes \mathbf{a}_{d}^{l},
\end{eqnarray}
or equivalently,
\begin{eqnarray}\label{CPdecompolyeq}
\mathcal{U}(x_1,x_2,\dots,x_d) = \sum_{l=1}^{r} \mathbf{a}_{1}^{l}(x_1)  \mathbf{a}_{2}^{l}(x_2) \dots  \mathbf{a}_{d}^{l}(x_d),
\end{eqnarray}
where $\otimes$ denotes the tensor product (outer product) and $\mathbf{a}_{i}^{l}(x_i) $ denotes the $x_i$-th entry of vector $\mathbf{a}_{i}^{l}$. Note that $\mathbf{a}_{1}^{l} \otimes \mathbf{a}_{2}^{l} \otimes \dots \otimes \mathbf{a}_{d}^{l} \in \mathbb{R}^{n_1  \times \cdots \times n_d}$ is a rank-$1$ tensor, $l=1,2,\dots,r$.

\subsubsection{Tucker Decomposition}\label{Tuckermodel}

Given $\mathcal{U} \in \mathbb{R}^{n_1  \times \cdots \times n_d}$ and $\mathbf{X} \in \mathbb{R}^{n_i \times n_i^{\prime}}$, the product $\mathcal{U^{\prime}}  \triangleq  \mathcal{U} \times_i \mathbf{X} \in \mathbb{R}^{n_1 \times \cdots \times n_{i-1} \times n_i^{\prime} \times n_{i+1} \times \cdots \times n_d}$ is defined as
\begin{eqnarray}
\mathcal{U^{\prime}} (x_1,\cdots,x_{i-1},k_i,x_{i+1},\cdots,x_d) \triangleq \sum_{x_i=1}^{n_i} \mathcal{U} (x_1,\cdots,x_{i-1},x_i,x_{i+1},\cdots,x_d) \mathbf{X}(x_i,k_i).
\end{eqnarray}
The Tucker rank  of a tensor $\mathcal{U}$ is defined as $\text{rank}_{\text{Tucker}} (\mathcal{U})= \underline{ r} =(m_1,\ldots,m_d)$ where $m_i = \text{rank}(\mathbf{U}_{(i)})$, i.e., the rank of the $i$-th matricization, $i=1,\dots,d$. The Tucker decomposition of $\mathcal{U}$ is given by
\begin{eqnarray}\label{tuckertu}
\mathcal{U}(\vec{x}) =  \sum_{k_1=1}^{m_1} \cdots \sum_{k_{d}=1}^{m_{d}} \mathcal{C}(k_1,\ldots,k_d) \mathbf{T}_1(k_1,x_1)  \dots \mathbf{T}_d(k_d,x_d),
\end{eqnarray}
or in short
\begin{eqnarray}\label{TuckTuck}
\mathcal{U} = \mathcal{C} \times_{i=1}^d \mathbf{T}_i,
\end{eqnarray}
where $\mathcal{C} \in \mathbb{R}^{m_1 \times \cdots \times m_d}$ is the core tensor and $\mathbf{T}_i \in \mathbb{R}^{m_i \times n_i}$ are $d$ orthogonal matrices.

\subsubsection{TT Decomposition}\label{TTmodel}

The separation or TT rank of a tensor is defined as $\text{rank}_{\text{TT}} (\mathcal{U})= \underline{ r} = (u_1,\ldots,u_{d-1})$ where $u_i = \text{rank}(\mathbf{\widetilde U}_{(i)})$, i.e., the rank of the $i$-th unfolding, $i=1,\dots,d-1$. Note that $u_i \leq \max \{N_i, \bar N_{i} \}$ in general and also $u_1$ is simply the conventional matrix rank when $d=2$. The TT decomposition of a tensor $\mathcal{U}$ is given by
\begin{eqnarray}\label{TTeq2}
\mathcal{U}(\vec{x}) =  \sum_{k_1=1}^{u_1} \cdots \sum_{k_{d-1}=1}^{u_{d-1}}  \mathcal{U}^{(1)}(x_1,k_1)  \left( \prod_{i=2}^{d-1} \mathcal{U}^{(i)}(k_{i-1},x_i,k_i) \right) \mathcal{U}^{(d)}(k_{d-1},x_d),
\end{eqnarray}
or in short
\begin{eqnarray}\label{TTeq1}
\mathcal{U} = \mathcal{U}^{(1)} \dots  \mathcal{U}^{(d)},
\end{eqnarray}
where the $3$-way tensors $\mathcal{U}^{(i)} \in \mathbb{R}^{u_{i-1} \times n_i \times u_{i}}$ for $i=2,\dots,d-1$ and matrices $\mathcal{U}^{(1)} \in \mathbb{R}^{n_1 \times u_1}$  and $\mathcal{U}^{(d)} \in \mathbb{R}^{u_{d-1} \times n_d}$ are the components of this decomposition.

For each matrix or tensor model, we assume that the true rank of $\mathbf{U}$ or $\mathcal{U}$ is $r^*$ or $\underline{ r}^*$ which is unknown, and also $\mathbf{U}$ or $\mathcal{U}$ is chosen generically from the corresponding manifold.

\subsection{Problem Statement}\label{prost}

For each one of the above data models, we are interested in obtaining the upper bound on the unknown scalar-rank $r^*$ or component-wise upper bound on the unknown vector-rank $\underline{r}^*$, deterministically based on the sampling pattern $\mathbf{\Omega}$ or $\Omega$ and the rank of a given completion. Also, we aim to provide such bound that holds with high probability based only on the sampling probability of the entries and the rank of a given completion. Moreover, for the single-view matrix model and CP-rank tensor model, where the rank is a scalar, we provide both deterministic and probabilistic conditions such that the unknown rank can be exactly determined.

\section{Scalar-Rank Cases}\label{semnw}

\subsection{Single-View Matrix} \label{svmsubs}

Previously, this problem has been treated in \cite{converse}, where strong assumptions including the existence of a completion with rank $r \leq r^*$  have been used. In this section, we provide an analysis that does not require such assumption and moreover our analysis can be extended to multi-view data and tensors in the following sections. Furthermore, we show the tightness of our theoretical bounds via numerical examples. 

\subsubsection{Deterministic Rank Analysis}
\

The following assumption will be used frequently in this subsection.

{\bf Assumption $A_r$}: Each column of the sampled matrix includes at least $r$ sampled entries.

Consider an arbitrary column of the sampled matrix $\mathbf{U} \left(:,i\right)$, where $i \in  \{1,\dots,n_2\}$. Let $l_i = N_{\mathbf{\Omega}} (\mathbf{U}\left(:,i\right))$ denote the number of observed entries in the $i$-th column of $\mathbf{U}$. Assumption $A_r$ results that $l_i \geq r$. 

We construct a binary valued matrix called {\bf constraint matrix} $\mathbf{\breve{\Omega}}_r$ based on $\mathbf{\Omega}$ and a given number $r$. Specifically, we construct $l_i - r$ columns with binary entries based on the locations of the observed entries in $\mathbf{U}\left(:,i\right)$ such that each column has exactly $r+1$ entries equal to one. Assume that $x_1, \dots, x_{l_i}$ are the row indices of all observed entries in this column. Let $\mathbf{\Omega}^{i}_r$ be the corresponding $n_1 \times (l_i - r)$ matrix to this column which is defined as the following: for any $j \in \{1,\dots,l_i-r\}$, the $j$-th column has the value $1$ in rows $\{x_1,\dots , x_{r},x_{r+j}\}$ and zeros elsewhere. Define the binary constraint matrix as $\mathbf{\breve{\Omega}}_r = \left[\mathbf{\Omega}^{1}_r|\mathbf{\Omega}^{2}_r\dots|\mathbf{\Omega}^{n_2}_{r} \right] \in \mathbb{R}^{n_1 \times K_r}$ \cite{charact}, where $K_r = N_{\mathbf{\Omega}} (\mathbf{U})-n_2 r$.



{\bf Assumption $B_r$}: There exists a submatrix{\footnote{Specified by a subset of rows and a subset of columns (not necessarily consecutive).}} $\mathbf{\breve{\Omega}}_r^{\prime} \in \mathbb{R}^{n_1 \times K}$ of $\mathbf{\breve{\Omega}}_r$ such that $K= n_1 r - r^2$ and for any $K^{\prime} \in \{1,2,\dots,K\}$ and any submatrix $\mathbf{\breve{\Omega}}_r^{\prime \prime} \in \mathbb{R}^{n_1 \times K^{\prime}}$ of $\mathbf{\breve{\Omega}}_r^{\prime}$ we have
\begin{eqnarray}\label{matfineqdan}
r f(\mathbf{\breve{\Omega}}_r^{\prime \prime}) - r^2 \geq K^{\prime},
\end{eqnarray}
where $f(\mathbf{\breve{\Omega}}_r^{\prime \prime})$ denotes the number of nonzero rows of $\mathbf{\breve{\Omega}}_r^{\prime \prime}$.

Note that exhaustive enumeration is needed in order to check whether or not Assumption $B_r$ holds. Hence, the deterministic analysis cannot be used in practice for large-scale data. However, it serves as the basis of the subsequent probabilistic analysis that will lead to a simple lower bound on the sampling probability such that Assumption $B_r$ holds with high probability, which is of practical value.

In the following, we restate Theorem $1$ in \cite{charact} which will be used later. 

\begin{lemma}\label{thmmat1}
For almost every $\mathbf{U}$, there are finitely many completions of the sampled matrix if and only if Assumptions $A_{r^*}$ and $B_{r^*}$ hold.
\end{lemma}

Recall that the true rank $r^*$  is assumed unknown. 

\begin{definition}
Let $\mathcal{S}_{\mathbf{\Omega}}$ denote the set of all natural numbers $r$ such that both Assumptions $A_{r}$ and $B_{r}$ hold.
\end{definition}

\begin{lemma}\label{sinvorderkam}
There exists a number $r_{\mathbf{\Omega}} $ such that $\mathcal{S}_{\mathbf{\Omega}} = \{1,2,\dots,r_{\mathbf{\Omega}} \}$.
\end{lemma}

\begin{proof}
Assume that $1 < r \leq \min \{n_1,n_2\}$ and $r \in \mathcal{S}_{\mathbf{\Omega}}$. It suffices to show $r-1 \in \mathcal{S}_{\mathbf{\Omega}}$. By contradiction, assume that $r-1 \notin \mathcal{S}_{\mathbf{\Omega}}$. Therefore, according to Lemma \ref{thmmat1}, there exist infinitely many completions of $\mathbf{U}$ of rank $r-1$. Consider the decomposition $\mathbf{U}= \mathbf{X} \mathbf{Y}$, where $\mathbf{X} \in \mathbb{R}^{n_1 \times (r-1)}$ and $\mathbf{Y} \in \mathbb{R}^{(r-1) \times n_2}$ are the matrices of variables. Then, each observed entry of $\mathbf{U}$ results in a polynomial in terms of the entries of $\mathbf{X}$ and $\mathbf{Y}$
\begin{eqnarray}\label{polysingvma}
\mathbf{U}(i,j) = \sum_{l=1}^{r-1} \mathbf{X}(i,l) \mathbf{Y}(l,j),
\end{eqnarray}
and let $\mathcal{P}$ denote the set of all such polynomials. Since there exist infinitely many completions of $\mathbf{U}$ of rank $r-1$, the maximum number of algebraically independent polynomials among all the polynomials in the set $\mathcal{P}$ is less than $(r-1)(n_1+n_2 - (r-1))$ which is the dimension of the manifold of $n_1\times n_2$ matrices of rank $r-1$ \cite{kiraly2}; since otherwise, accroding to Bernstein's theorem \cite{bernstein}, there are at most finitely many completions. Hence, as we have $r \leq \min\{n_1,n_2\}$ and thus $(r-1)(n_1+n_2 - (r-1)) \leq r(n_1+n_2 - r)$, the maximum number of algebraically independent polynomials in $\mathcal{P}$ is less than $r(n_1+n_2 - r)$ as well, which is the dimension of the manifold of $n_1\times n_2$ matrices of rank $r$. Therefore, accroding to Bernstein's theorem with probability one, there exist infinitely many completions of the sampled matrix of rank $r$ and this contradicts the assumption.
\end{proof}

The following theorem provides a relationship between the unknown rank $r^*$ and $r_{\mathbf{\Omega}}$.

\begin{theorem}\label{thmmat2}
With probability one, exactly one of the following statements holds

(i) $r^* \in \mathcal{S}_{\mathbf{\Omega}} = \{1,2,\dots,r_{\mathbf{\Omega}} \}${\rm ;}

(ii) For any arbitrary completion of the sampled matrix $\mathbf{U}$ of rank $r$, we have $r \notin \mathcal{S}_{\mathbf{\Omega}}$.
\end{theorem}

\begin{proof}
Suppose that there does not exist a completion of the sampled matrix $\mathbf{U}$ of rank $r$ such that  $r \in \mathcal{S}_{\mathbf{\Omega}}$. Therefore, it is easily verified that statement (ii) holds and statement (i) does not hold. On the other hand, assume that there exists a completion of the sampled matrix $\mathbf{U}$ of rank $r$, where  $r \in \mathcal{S}_{\mathbf{\Omega}}$. Hence, statement (ii) does not hold and to complete the proof it suffices to show that with probability one, statement (i) holds.

Observe that $r_{\mathbf{\Omega}} \in \mathcal{S}_{\mathbf{\Omega}}$, and therefore Assumption $A_{r_{\mathbf{\Omega}}}$ holds. Hence, each column of $\mathbf{U}$ includes at least $r_{\mathbf{\Omega}}+1$ observed entries. On the other hand, the existence of a completion of the sampled matrix $\mathbf{U}$ of rank $r \in  \mathcal{S}_{\mathbf{\Omega}}$ results in the existence of a basis $\mathbf{X} \in \mathbb{R}^{n_1 \times r}$ such that each column of $\mathbf{U}$ is a linear combination of the columns of $\mathbf{X}$, and thus there exists $\mathbf{Y} \in \mathbb{R}^{r \times n_2}$ such that $\mathbf{U}_{\mathbf{\Omega}} = \left( \mathbf{XY} \right)_{\mathbf{\Omega}}$. Hence, given $\mathbf{X}$, each observed entry $\mathbf{U}(i,j)$ results in a degree-$1$ polynomial in terms of the entries of $\mathbf{Y}$ as the following
\begin{eqnarray}\label{polysingvmacopor}
\mathbf{U}(i,j) = \sum_{l=1}^{r} \mathbf{X}(i,l) \mathbf{Y}(l,j).
\end{eqnarray}

Consider the first column of $\mathbf{U}$ and recall that it includes at least $r_{\mathbf{\Omega}}+1 \geq r+1$ observed entries. The genericity of the coefficients of the above-mentioned polynomials results that using $r$ of the observed entries the first column of $\mathbf{Y}$ can be determined uniquely. This is because there exists a unique solution for a system of $r$ linear equations in $r$ variables that are linearly independent. Then, there exists at least one more observed entry besides these $r$ observed entries in the first column of $\mathbf{U}$ and it can be written as a linear combination of the $r$ observed entries that have been used to obtain the first column of $\mathbf{Y}$. Let $\mathbf{U}(i_1,1)$, $\dots$, $\mathbf{U}(i_r,1)$ denote the $r$ observed entries that have been used to obtain the first column of $\mathbf{Y}$ and $\mathbf{U}(i_{r+1},1)$ denote the other observed entry. Hence, the existence of a completion of the sampled matrix $\mathbf{U}$ of rank $r \in  \mathcal{S}_{\mathbf{\Omega}}$ results in an equation as the following
\begin{eqnarray}\label{polysingvmalindep}
\mathbf{U}(i_{r+1},1) = \sum_{l=1}^{r} t_l \mathbf{U}(i_l,1),
\end{eqnarray}
where $t_l$'s are constant scalars, $l=1,\dots,r$. Assume that $r^* \notin \mathcal{S}_{\mathbf{\Omega}}$, i.e., statement (i) does not hold. Then, note that $r^* \geq r+1$ and $\mathbf{U}$ is chosen generically from the manifold of $n_1 \times n_2$ rank-$r^*$ matrices, and therefore an equation of the form of \eqref{polysingvmalindep} holds with probability zero. Moreover, according to Lemma \ref{thmmat1} there exist at most finitely many completions of the sampled matrix of rank $r$. Therefore, there exist a completion of $\mathbf{U}$ of rank $r$ with probability zero, which contradicts the initial assumption that there exists a completion of the sampled matrix $\mathbf{U}$ of rank $r$, where  $r \in \mathcal{S}_{\mathbf{\Omega}}$.
\end{proof}

\begin{corollary}\label{colthmmatsv1}
Consider an arbitrary number $r^{\prime} \in \mathcal{S}_{\mathbf{\Omega}}$. Similar to Theorem \ref{thmmat2}, it follows that with probability one, exactly one of the followings holds

(i) $r^* \in  \{1,2,\dots,r^{\prime} \}${\rm ;}

(ii) For any arbitrary completion of the sampled matrix $\mathbf{U}$ of rank $r$, we have $r \notin \{1,2,\dots,r^{\prime} \}$.
\end{corollary}

As a result of Corollary \ref{colthmmatsv1}, we have the following.

\begin{corollary}\label{colthmmatsv1nw}
Assuming that there exists a rank-$r$ completion of the sampled matrix $\mathbf{U}$ such that $r \in \mathcal{S}_{\mathbf{\Omega}}$, then with probability one $r^* \leq r$.
\end{corollary}

\begin{corollary}\label{colthmmatsvopt1}
Let $\mathbf{U}^*$ denote an optimal solution to the following NP-hard optimization problem
{\rm
\begin{align}\label{optpro}
& \ \ \ \ \ \  \ \ \ \ \ \  \   \ \text{minimize}_{\mathbf{U}^{\prime} \in \mathbb{R}^{n_1 \times n_2}} 
& &  \text{rank}(\mathbf{U}^{\prime}) \ \ \ \ \  \ \ \  \ \ \  \ \ \  \ \ \  \ \ \  \ \ \  \ \ \  \ \ \   \\
& \ \ \ \ \ \  \ \ \ \ \ \  \   \ \text{subject to}
& & \mathbf{U}^{\prime}_{\mathbf{\Omega}} = \mathbf{U}_{\mathbf{\Omega}}. \nonumber
\end{align}
}
Also, let $\mathbf{\hat U}$ denote a suboptimal solution to the above optimization problem. Then, Corollary \ref{colthmmatsv1} results the following statements:

(i) If {\rm $\text{rank}(\mathbf{U}^*) \in \mathcal{S}_{\mathbf{\Omega}}$}, then {\rm $r^* = \text{rank}(\mathbf{U}^*)$ }with probability one.

(ii) If {\rm $\text{rank}(\mathbf{\hat U}) \in \mathcal{S}_{\mathbf{\Omega}}$}, then {\rm $r^* \leq \text{rank}(\mathbf{\hat U})$} with probability one.

\end{corollary}

\begin{remark}
One challenge of applying Corollary \ref{colthmmatsvopt1} or any of the other obtained deterministic  results is the computation of $\mathcal{S}_{\mathbf{\Omega}}$, which involves exhaustive enumeration to check Assumption $B_r$. Next, for each number $r$, we provide a lower bound on the sampling probability in terms of $r$ that ensures $r \in \mathcal{S}_{\mathbf{\Omega}}$ with high probability. Consequently, we do not need to compute $\mathcal{S}_{\mathbf{\Omega}}$ but instead we can certify the above results with high probability.
\end{remark}

\subsubsection{Probabilistic Rank Analysis} 
\


The following lemma is a re-statement of Theorem $3$ in \cite{charact}, which is the probabilistic version of Lemma \ref{thmmat1}.

\begin{lemma}\label{danthmsingviwmat}
Suppose $r \leq \frac{n_1}{6}$ and that each {\bf column} of the sampled matrix is observed in at least $l$ entries, uniformly at random and independently across entries, where
\begin{eqnarray}\label{genmatrix}
l > \max\left\{12 \ \log \left( \frac{n_1}{\epsilon} \right) + 12, 2r\right\}. 
\end{eqnarray}
Also, assume that $ r(n_1-r) \leq n_2$. Then, with probability at least $1 - \epsilon$, $r \in \mathcal{S}_{\mathbf{\Omega}}$.
\end{lemma}







The following lemma is taken from \cite{ashraphijuo} and will be used to derive a lower bound on the sampling probability that leads to  the similar statement as Theorem \ref{thmmat2} with high probability.

\begin{lemma}\label{azumares}
Consider a vector with $n$ entries where each entry is observed with  probability  $p$  independently from the other entries. If $p > p^{\prime} = \frac{k}{n} + \frac{1}{\sqrt[4]{n}}$, then with probability  at least $\left(1-\exp(-\frac{\sqrt{n}}{2})\right)$, more than $k$ entries are observed.
\end{lemma}

The following proposition characterizes the probabilistic version of Theorem \ref{thmmat2}. 

\begin{proposition}\label{thmprobsvmsamppr}
Suppose $r \leq \frac{n_1}{6}$, $r(n_1 -r) \leq n_2$ and that each entry of the sampled matrix is observed  uniformly at random and independently across entries with probability $p$, where
\begin{eqnarray}\label{genmatrixrsamppro}
p > \frac{1}{n_1} \max\left\{12 \ \log \left( \frac{n_1}{\epsilon} \right) + 12, 2r\right\} + \frac{1}{\sqrt[4]{n_1}}. 
\end{eqnarray}
Then, with probability at least $\left( 1 - \epsilon \right)\left(1-\exp(-\frac{\sqrt{n_1}}{2})\right)^{n_2}$, we have $r \in \mathcal{S}_{\mathbf{\Omega}}$. 
%
%
%
\end{proposition}

\begin{proof}
Consider an arbitrary column of $\mathbf{U}$ and note that resulting from Lemma \ref{azumares} the number of observed entries at this column of $\mathbf{U}$  is greater than $\max\left\{12 \ \log \left( \frac{n_1}{\epsilon} \right) + 12, 2r\right\}$ with probability at least $\left(1-\exp(-\frac{\sqrt{n_1}}{2})\right)$. Therefore, the number of sampled entries at each column satisfies
\begin{eqnarray}\label{genmatrixr}
l > \max\left\{12 \ \log \left( \frac{n_1}{\epsilon} \right) + 12, 2r\right\}, 
\end{eqnarray}
with probability at least $\left(1-\exp(-\frac{\sqrt{n_1}}{2})\right)^{n_2}$. Thus, resulting from Lemma \ref{danthmsingviwmat} with probability at least \\ $\left( 1 - \epsilon \right)\left(1-\exp(-\frac{\sqrt{n_1}}{2})\right)^{n_2}$, we have $r \in \mathcal{S}_{\mathbf{\Omega}}$.
\end{proof}

Finally, we have the following probabilistic version of Corollary \ref{colthmmatsvopt1}.


\begin{corollary}\label{colthmmatsv1nwx}
Assume that {\rm $\text{rank}(\mathbf{U}^*) \leq \frac{n_1}{6}$} and {\rm $\text{rank}(\mathbf{U}^*)(n_1 -\text{rank}(\mathbf{U}^*)) \leq n_2$} and \eqref{genmatrixrsamppro} holds for {\rm $r=\text{rank}(\mathbf{U}^*)$}, where $\mathbf{U}^*$ denotes an optimal solution to the optimization problem \eqref{optpro}. Then, according to Proposition \ref{thmprobsvmsamppr} and Corollary \ref{colthmmatsvopt1}, with probability at least $\left( 1 - \epsilon \right)\left(1-\exp(-\frac{\sqrt{n_1}}{2})\right)^{n_2}$, {\rm $r^* = \text{rank}(\mathbf{U}^*)$}. Similarly, assume that {\rm $\text{rank}(\mathbf{\hat U}) \leq \frac{n_1}{6}$} and {\rm $\text{rank}(\mathbf{\hat U})(n_1 -\text{rank}(\mathbf{\hat U})) \leq n_2$} and \eqref{genmatrixrsamppro} holds for {\rm $r=\text{rank}(\mathbf{\hat U})$}, where $\mathbf{\hat U}$ denotes a suboptimal solution to the  optimization problem \eqref{optpro}. Then, with probability at least $\left( 1 - \epsilon \right)\left(1-\exp(-\frac{\sqrt{n_1}}{2})\right)^{n_2}$, {\rm $r^* \leq \text{rank}(\mathbf{\hat U})$}.
\end{corollary}

\subsubsection{Numerical Results}\label{simusec}
\

In Fig. \ref{fig1} and Fig. \ref{fig2}, the x-axis represents the sampling probability, and the y-axis denotes the value of $r$. The color scale represents the lower bound on the probability of event $r \in \mathcal{S}_{\mathbf{\Omega}}$. For example, as we can observe in Fig. \ref{fig1}, for any $r \in \{1,\dots,44\}$ we have $r \in \mathcal{S}_{\mathbf{\Omega}}$ with probability at least $0.6$ (approximately based on the color scale since the corresponding points are orange) given that $p =0.54$.

We consider the sampled matrix $\mathbf{U} \in \mathbb{R}^{300 \times 15000}$ and $\mathbf{U} \in \mathbb{R}^{1200 \times 240000}$ in Fig. \ref{fig1} and Fig. \ref{fig2}, respectively. In particular, for fixed values of sampling probability $p$ and $r$, we first find a ``small'' $\epsilon$ that \eqref{genmatrixrsamppro} holds by trial-and-error. Then, according to Proposition \ref{thmprobsvmsamppr}, we conclude that with probability at least $\left( 1 - \epsilon \right)\left(1-\exp(-\frac{\sqrt{n_1}}{2})\right)^{n_2}$,  $r \in \mathcal{S}_{\mathbf{\Omega}}$.

\begin{figure}[h]
	\centering
		{\includegraphics[width=12cm]{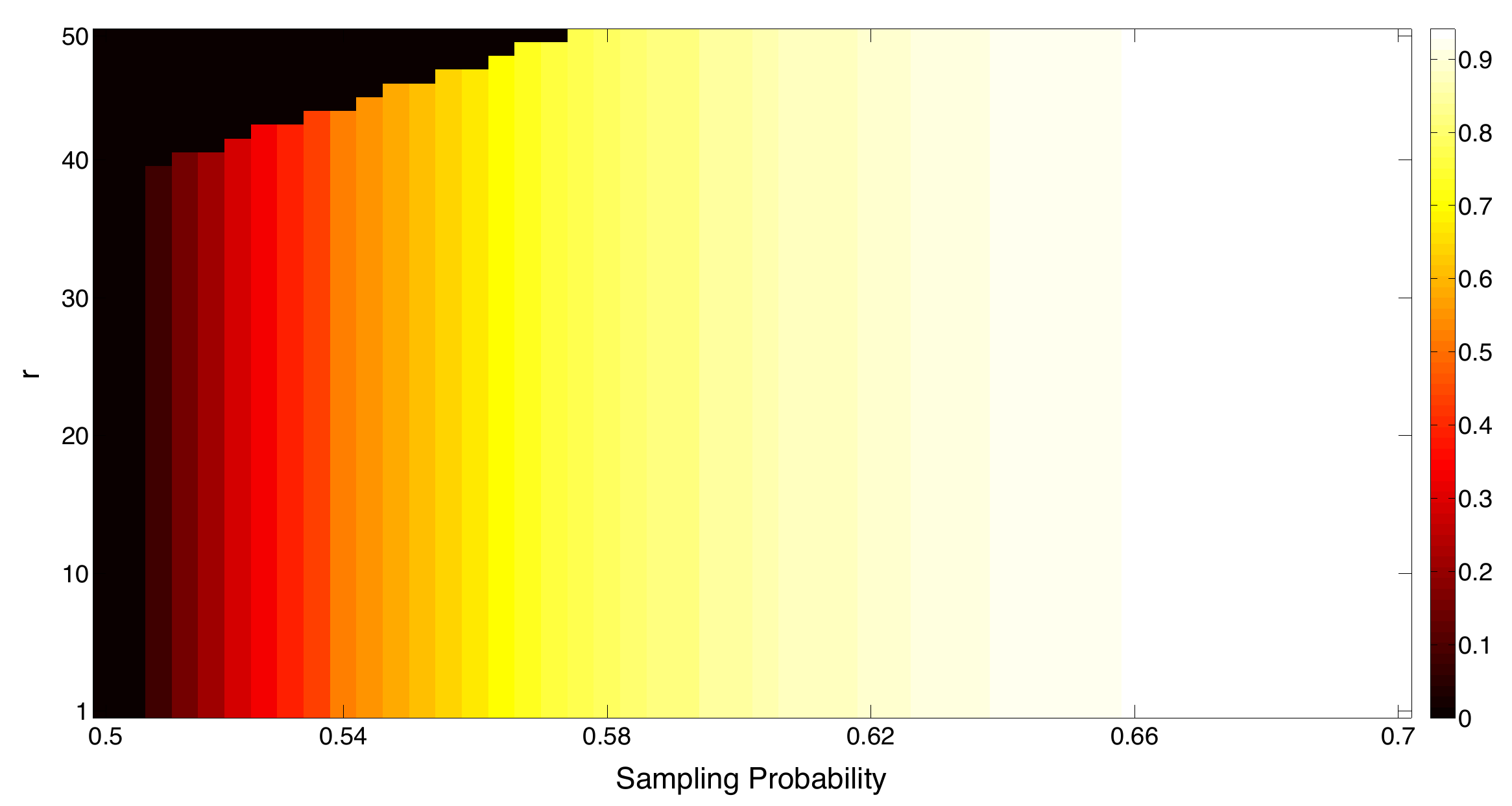}}
	\caption{ Probability of $r \in \mathcal{S}_{\mathbf{\Omega}}$ as a function of sampling probability for $\mathbf{U} \in \mathbb{R}^{300 \times 15000}$.}
	\label{fig1}\vspace{-4mm}
\end{figure}

\begin{figure}[h]
	\centering
		{\includegraphics[width=12cm]{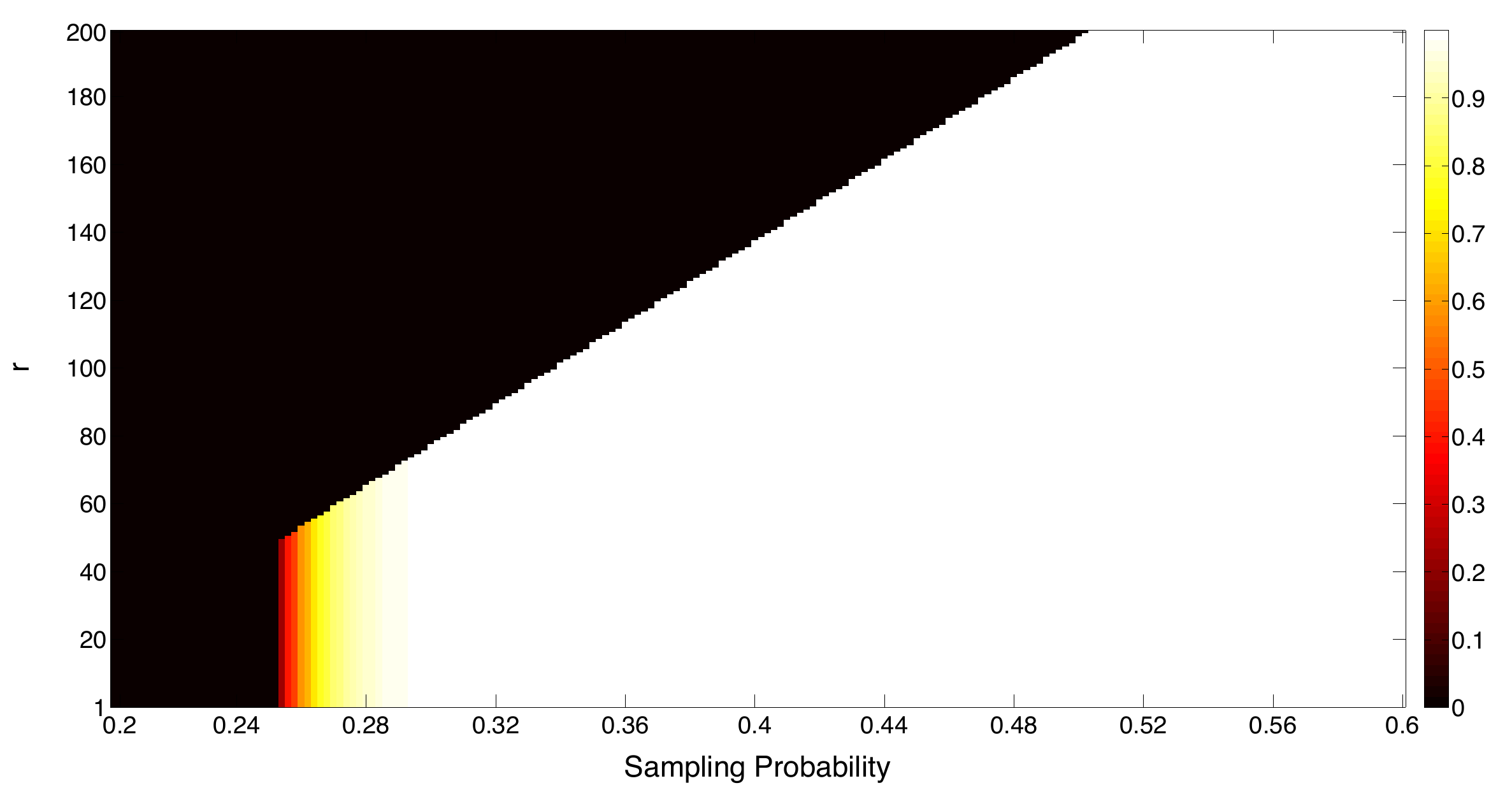}}
	\caption{ Probability of $r \in \mathcal{S}_{\mathbf{\Omega}}$ as a function of sampling probability for $\mathbf{U} \in \mathbb{R}^{1200 \times 240000}$.}
	\label{fig2}\vspace{-4mm}
\end{figure}

The purpose of Fig. \ref{fignewcomp} is to show how tight our proposed upper bounds on rank can be. Here, we first generate an $n_1 \times n_2$ random matrix of a given rank $r$ by multiplying a random (entries are drawn according to a uniform distribution on real numbers within an interval) $n_1 \times r$ matrix and $r \times n_2$ matrix. Then, each entry of the randomly generated matrix is sampled  uniformly at random and independently across entries with some sampling probability $p$. Afterwards, we apply the nuclear norm minimization method proposed in \cite{cant2} for matrix completion, where the non-convex objective function in \eqref{optpro} is relaxed by using nuclear norm, which is the convex hull of the rank function, as follows
\begin{align}\label{optprorelaxedmat}
& \ \ \ \ \ \  \ \ \ \ \ \  \   \ \text{minimize}_{\mathbf{U}^{\prime} \in \mathbb{R}^{n_1 \times n_2}} 
& &  \|\mathbf{U}^{\prime}\|_* \ \ \ \ \  \ \ \  \ \ \  \ \ \  \ \ \  \ \ \  \ \ \  \ \ \  \ \ \   \\
& \ \ \ \ \ \  \ \ \ \ \ \  \   \ \text{subject to}
& & \mathbf{U}^{\prime}_{\mathbf{\Omega}} = \mathbf{U}_{\mathbf{\Omega}}, \nonumber
\end{align}
where $\|\mathbf{U}^{\prime}\|_*$ denotes the nuclear norm of $\mathbf{U}^{\prime}$. Let $\mathbf{\hat U}^*$ denote an optimal solution to \eqref{optprorelaxedmat} and recall that $\mathbf{U}^*$ denotes an optimal solution to \eqref{optpro}. Since \eqref{optprorelaxedmat} is a convex relaxation to \eqref{optpro}, we conclude that $\mathbf{\hat U}^*$ is a suboptimal solution to \eqref{optpro}, and therefore $\text{rank}(\mathbf{ U}^*) \leq \text{rank}(\mathbf{ \hat U}^*)$. We used the Matlab program found online \cite{matlabcode} to solve \eqref{optprorelaxedmat}.

As an example, we generate a random matrix $\mathbf{U} \in \mathbb{R}^{300 \times 15000}$ (the same size as the matrix in Fig. \ref{fig1}) of rank $r$ as described above for $r \in \{1,\dots,50\}$ and some values of the sampling probability $p$. Then, we obtain the rank of the completion given by \eqref{optprorelaxedmat} and denote it by $r^{\prime}$. Due to the randomness of the sampled matrix, we repeat this procedure $5$ times. We calculate the ``gap'' $r^{\prime} -r $ in each of these $5$ runs and denote the maximum and minimum among these 5 numbers by $d_{\text{max}}$ and $d_{\text{min}}$, respectively. Hence, $d_{\text{max}}$ and $d_{\text{min}}$ represent the loosest (worst) and tightest (best) gaps between the rank obtained by \eqref{optprorelaxedmat} and rank of the original sampled matrix  over $5$ runs, respectively. In Fig. \ref{fignewcomp}, the maximum and minimum gaps are plotted as a function of rank of the matrix, for different sampling probabilities.


\begin{figure}[htbp]
\centering
\subfigure[$p = 0.46$.]{
	\includegraphics[width=5.9cm]{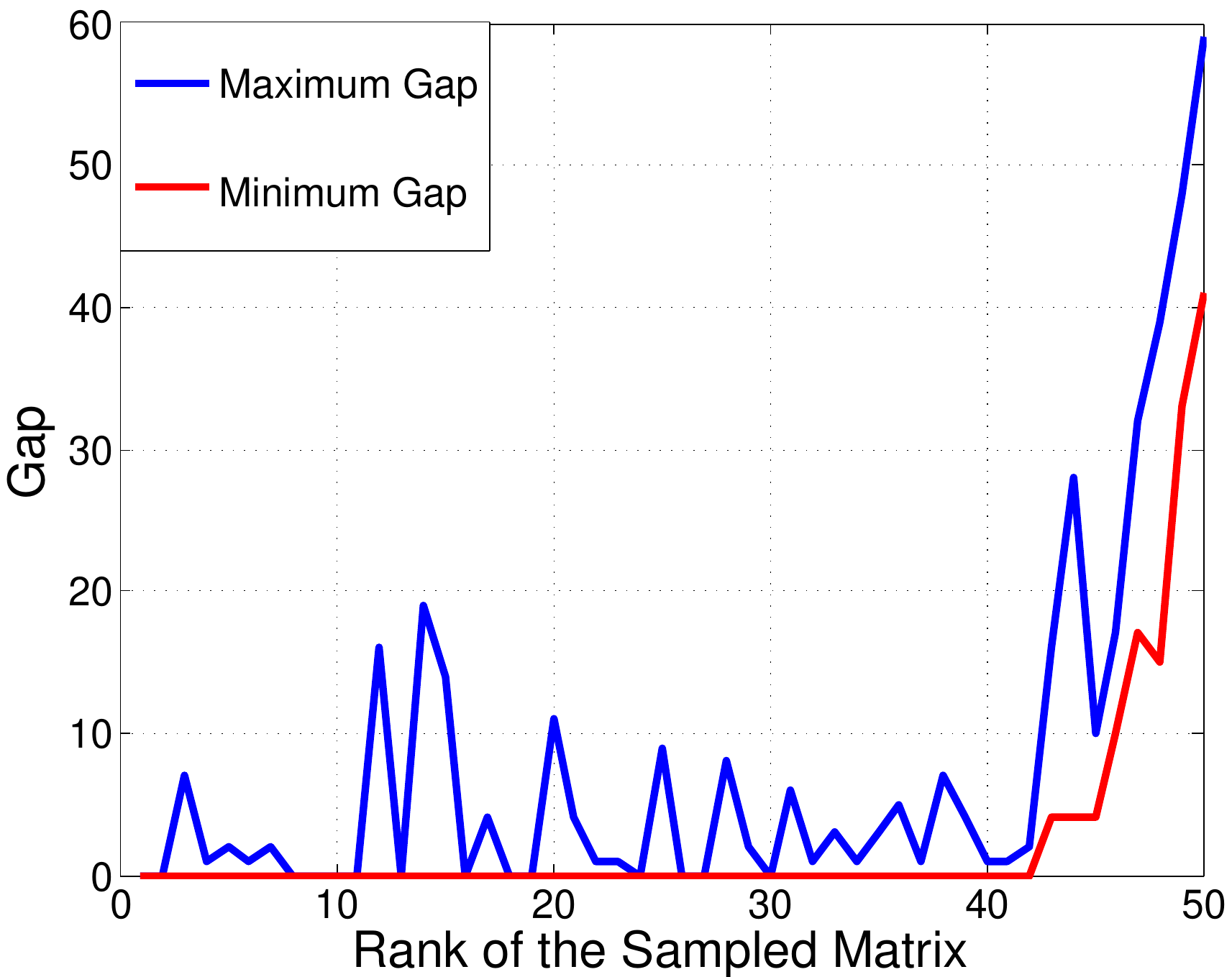}
\label{fig:subfigS4.1}
}
\subfigure[$p = 0.50$.]{
	\includegraphics[width=5.9cm]{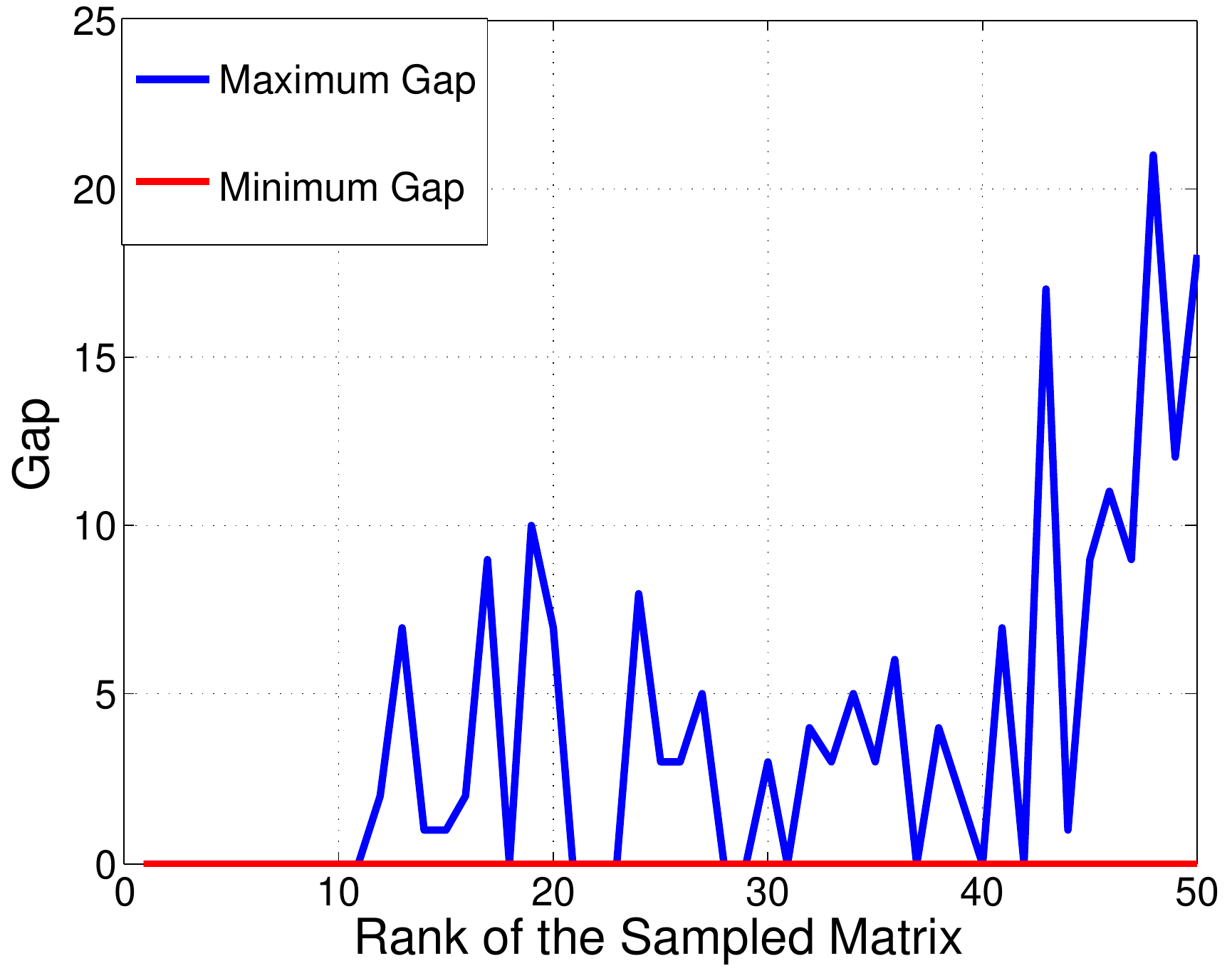}
\label{fig:subfigS4.2}
}
\subfigure[$p = 0.54$.]{
	\includegraphics[width=5.9cm]{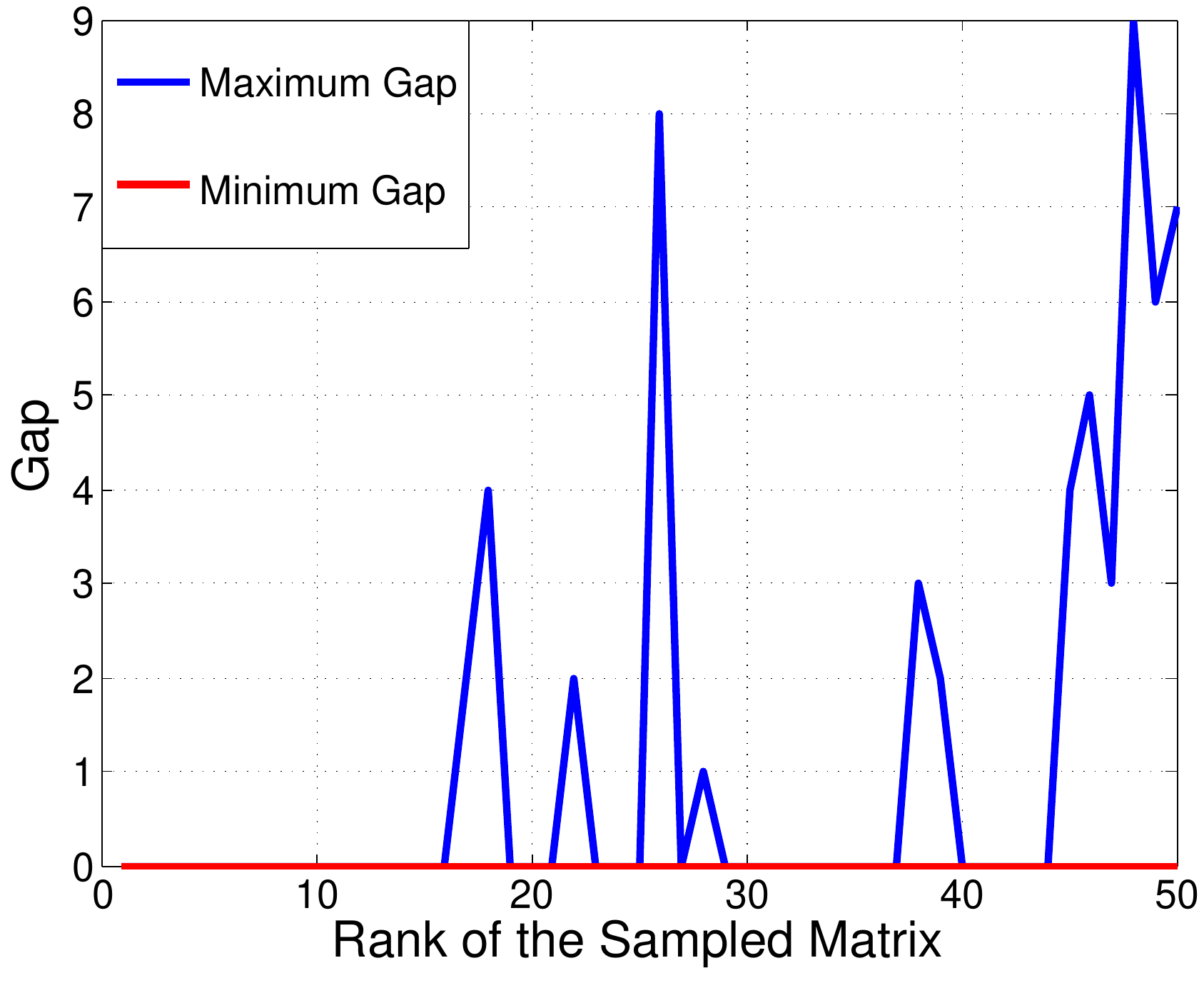}
\label{fig:subfigS4.3}
}
\subfigure[$p = 0.58$.]{
	\includegraphics[width=5.9cm]{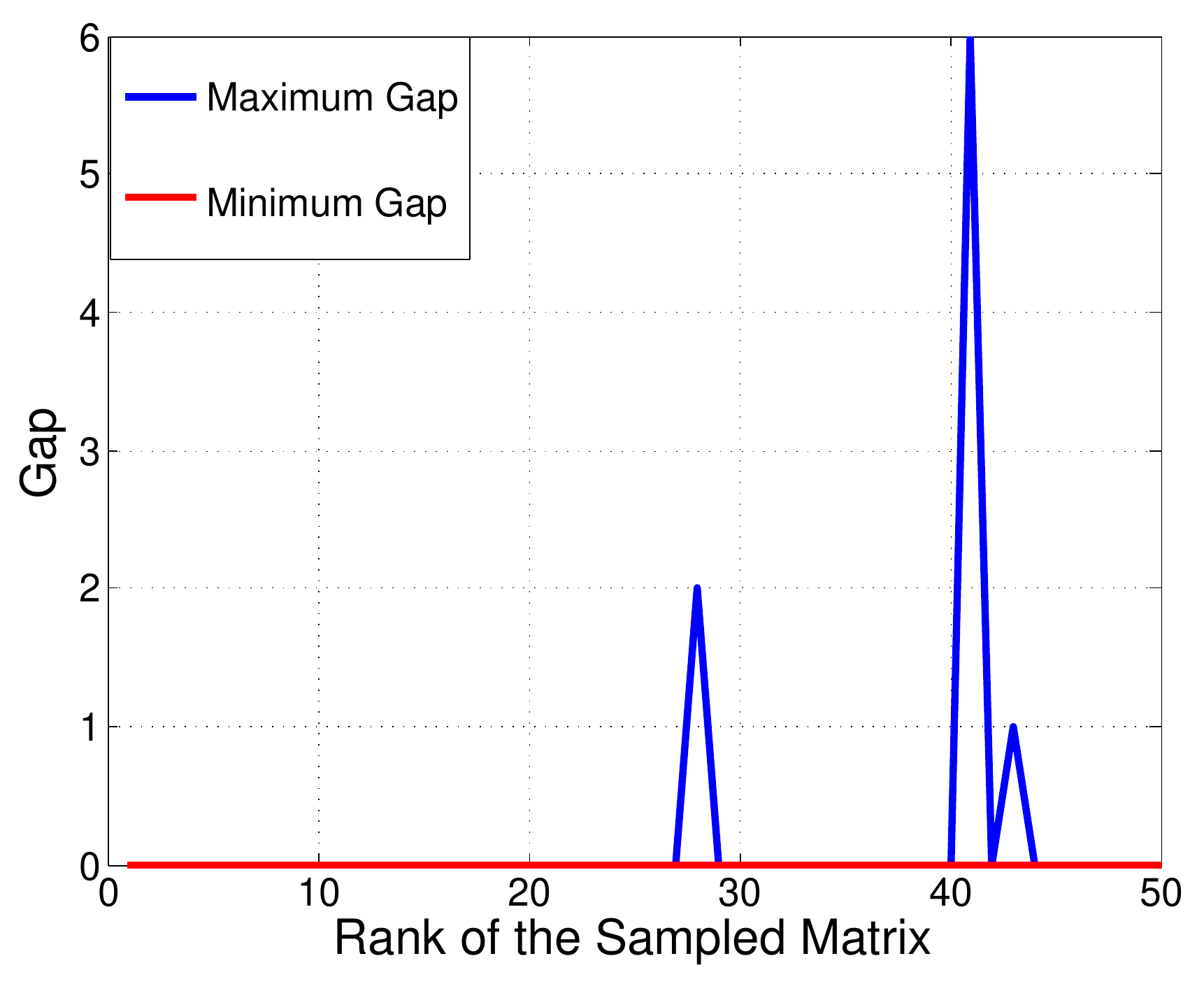}
\label{fig:subfigS4.4}
}
\caption[Optional caption for list of figures]{The gaps between the rank of the obtained matrix via \eqref{optprorelaxedmat} and that of the original sampled matrix.}
\label{fignewcomp}
\end{figure}

We have the following observations.

\begin{itemize}
\item According to Fig. \ref{fig1}, for $p= 0.54$ and $p=0.58$ we can ensure that the rank of any completion is an upper bound on the rank of the sampled matrix or $r^*$ with probability at least $0.6$ and $0.8$, respectively.
\item As we can observe in Fig. \ref{fig:subfigS4.1}-(d), the defined gap is always a nonnegative number, which is consistent with previous  observation that for $p= 0.54$ and $p=0.58$ we can certify that with high probability ($\geq 0.6$) the rank of any completion is an upper bound on the rank of the sampled matrix or $r^*$. 
\item For $p= 0.54$ and $p=0.58$ that we have theoretical results (as mentioned in the first observation) the gap obtained by \eqref{optprorelaxedmat} is very close to zero. This phenomenon (that we do not have a rigorous justification for) shows that as soon as we can certify our proposed theoretical results (i.e., as soon as the rank of a completion provides an upper bound on the rank of the sampled matrix or $r^*$) by increasing the sampling probability, the upper bound found through \eqref{optprorelaxedmat} becomes very tight; in some cases this bound is exactly equal to $r^*$ (red curves) and in some cases this bound is almost equal to $r^*$ (blue curves). However, these gaps are not small (specially blue curves) for $p= 0.46$ and $p=0.50$ and note that according to Fig. \ref{fig1}, for these values of $p$ we cannot guarantee the bounds on the value of rank hold with high probability.
\end{itemize}

\subsection{CP-Rank Tensor} \label{CPtensubs}

In this subsection, we assume that the sampled tensor $\mathcal{U} \in \mathbb{R}^{n_1 \times \dots \times n_d}$ is chosen generically from the manifold of tensors of rank $r^*=\text{rank}_{\text{CP}} (\mathcal{U})$, where $r^*$ is unknown.


{\bf Assumption $\mathcal{A}_r$}: Each row of the $d$-th matricization of the sampled tensor, i.e., $\mathbf{U}_{(d)}$ includes at least $r$ observed entries.


We construct a binary valued tensor called {\bf constraint tensor} ${\breve{\Omega}}_r$ based on ${\Omega}$ and a given number $r$. Consider any subtensor $\mathcal{Y} \in \mathbb{R}^{n_1 \times n_2 \times \cdots \times n_{d-1} \times 1 }$ of the tensor $\mathcal{U}$. The sampled tensor $\mathcal{U}$ includes $n_d$ subtensors that belong to  $\mathbb{R}^{n_1 \times n_2 \times \cdots \times n_{d-1} \times 1 }$ and let $\mathcal{Y}_i$ for $1 \leq i \leq n_d$ denote these $n_d$ subtensors. Define a binary valued tensor $\mathcal{\breve{Y}}_{i} \in \mathbb{R}^{n_1 \times n_2 \times \cdots \times n_{d-1} \times  k_i}$, where $k_i= N_{\Omega}(\mathcal{Y}_{i}) - r$ and its entries are described as the following. We can look at $\mathcal{\breve{Y}}_{i}$ as $k_i$ tensors each belongs to $\mathbb{R}^{n_1 \times n_2 \times \cdots \times n_{d-1} \times 1 }$. For each of the mentioned $k_i$ tensors in $\mathcal{\breve{Y}}_{i}$ we set the entries corresponding to $r$ of the observed entries equal to $1$. For each of the other $k_i$ observed entries, we pick one of the $k_i$ tensors of $\mathcal{\breve{Y}}_{i}$ and set its corresponding entry (the same location as that specific observed entry) equal to $1$ and set the rest of the entries equal to $0$. In the case that $k_i=0$ we simply ignore $\mathcal{\breve{Y}}_{i}$, i.e., $\mathcal{\breve{Y}}_{i} = \emptyset$

By putting together all $n_d$ tensors in dimension $d$, we construct a binary valued tensor ${\breve{\Omega}}_r \in \mathbb{R}^{n_1 \times n_2 \times \cdots \times n_{d-1} \times K}$, where $K = \sum_{i=1}^{n_d} k_i = N_{\Omega}(\mathcal{U}) - r n_d$ and call it the {\bf constraint tensor}. Observe that each subtensor of ${\breve{\Omega}}_r$ which belongs to $\mathbb{R}^{n_1 \times n_2 \times \cdots \times n_{d-1} \times 1 }$ includes exactly $r+1$ nonzero entries. In \cite{ashraphijuo4}, an example is given on the construction of ${\breve{\Omega}}_r$.

{\bf Assumption $\mathcal{B}_r$}: ${\breve{\Omega}}_r$ consists a subtensor ${\breve{\Omega}}_r^{\prime} \in \mathbb{R}^{n_1 \times n_2 \times \cdots \times n_{d-1} \times K}$ such that $K= r(\sum_{i=1}^{d-1} n_i) - r^2 - r(d-2)$ and for any $K^{\prime} \in \{1,2,\dots,K\}$ and any subtensor ${\breve{\Omega}}_{r}^{\prime \prime} \in \mathbb{R}^{n_1 \times n_2 \times \cdots \times n_{d-1} \times K^{\prime}}$ of ${\breve{\Omega}}_{r}^{\prime}$ we have
\begin{eqnarray}\label{matfineqdanmvCP}
r \left( \left( \sum_{i=1}^{d-1} f_i (\breve{\Omega}^{\prime \prime}_r) \right) -   \text{min} \left\{ \text{max} \left\{ f_1 (\breve{\Omega}^{\prime \prime}_r) , \dots ,f_{d-1}(\breve{\Omega}^{\prime \prime}_r) \right\} ,r \right\} - (d-2) \right) \geq K^{\prime},
\end{eqnarray}
where $f_i (\breve{\Omega}^{\prime \prime}_r)$ denotes the number of nonzero rows of the $i$-th matricization of $\breve{\Omega}^{\prime \prime}_r$.


The following lemma is a re-statement of Theorem $1$ in \cite{ashraphijuo4}.

\begin{lemma}\label{thmten1CP}
For almost every $\mathcal{U}$, there are only finitely many rank-$r^*$ completions of the sampled tensor if and only if Assumptions $\mathcal{A}_{r^*}$ and $\mathcal{B}_{r^*}$ hold.
\end{lemma}

\begin{definition}
Let $\mathcal{S}_{{\Omega}}$ denote the set of all natural numbers $r$ such that both Assumptions $\mathcal{A}_{r}$ and $\mathcal{B}_{r}$ hold.
\end{definition}

\begin{lemma}\label{tensCPorderkam}
There exists a number $r_{{\Omega}} $ such that $\mathcal{S}_{{\Omega}} = \{1,2,\dots,r_{{\Omega}} \}$.
\end{lemma}

\begin{proof}
The proof is similar to the proof of Lemma \ref{sinvorderkam} with the only difference that the dimension of the manifold of CP rank-$r$ tensors is $r(\sum_{i=1}^{d} n_i) - r^2 - r(d-1)$ \cite{ashraphijuo4}, which is an increasing function in $r$.
\end{proof}

The following theorem gives an upper bound on the unknown rank $r^*$.

\begin{theorem}\label{thmtenCPr2}
For almost every $\mathcal{U}$, with probability one, exactly one of the following statements holds

(i) $r^* \in \mathcal{S}_{{\Omega}} = \{1,2,\dots,r_{{\Omega}} \}${\rm ;}

(ii) For any arbitrary completion of the sampled tensor $\mathcal{U}$ of rank $r$, we have $r \notin \mathcal{S}_{{\Omega}}$.
\end{theorem}

\begin{proof}
Similar to the proof of Theorem \ref{thmmat2}, it suffices to show that the assumption $r^* \notin \mathcal{S}_{{\Omega}}$ results that there exists a completion of $\mathcal{U}$ of CP rank $r$, where  $r \in \mathcal{S}_{{\Omega}}$, with probability zero. Define $\mathcal{V} = (\mathcal{V}_1,\dots , \mathcal{V}_r)$ as the basis of the rank-$r$ CP decomposition of $\mathcal{U}$ as in \eqref{CPdecom}, where $\mathcal{V}_l = \mathbf{a}_{1}^{l} \otimes \mathbf{a}_{2}^{l} \otimes \dots \otimes \mathbf{a}_{d-1}^{l} \in \mathbb{R}^{n_1 \times \dots n_{d-1}}$ is a rank-$1$ tensor and $\mathbf{a}_{i}^{l}$ is defined in  \eqref{CPdecom} for $1 \leq l \leq r$ and $ 1 \leq i \leq d$. Define $\mathcal{Y} = (\mathbf{a}_{d}^{1},\dots,\mathbf{a}_{d}^{r})$ and $\mathcal{V} \otimes_d \mathcal{Y} = \sum_{l=1}^{r} \mathcal{V}_l \otimes \mathbf{a}_{d}^{l}$. Observe that $\mathcal{U} = \sum_{l=1}^{r} \mathcal{V}_l \otimes \mathbf{a}_{d}^{l} = \mathcal{V} \otimes_d \mathcal{Y}$.

Observe that each row of $\mathbf{U}_{(d)}$ includes at least $r_{{\Omega}}+1$ observed entries since Assumption $\mathcal{A}_{r_{{\Omega}}}$ holds. Moreover, the existence of a completion of the sampled tensor $\mathcal{U}$ of rank $r \in  \mathcal{S}_{{\Omega}}$ results in the existence of a basis $\mathcal{V} =  (\mathcal{V}_1,\dots , \mathcal{V}_r)$ such that there exists $\mathcal{Y} =  (\mathbf{a}_{d}^{1},\dots,\mathbf{a}_{d}^{r})$ and $\mathcal{U}_{\Omega} = \left( \mathcal{V} \otimes_d \mathcal{Y} \right)_{\Omega}$. As a result, given $\mathcal{V}$, each observed entry of $\mathcal{U}$ results in a degree-$1$ polynomial in terms of the entries of $\mathcal{Y}$ as 
\begin{eqnarray}\label{polytensorCPdecopor}
\mathcal{U}(\vec{x}) = \sum_{l=1}^{r} \mathcal{V}_l(x_1,\dots,x_{d-1}) \mathbf{a}_{d}^{l}(x_d).
\end{eqnarray}

Note that $r_{{\Omega}} \geq r$ and each row of $\mathbf{U}_{(d)}$ includes at least $r_{{\Omega}}+1 \geq r+1$ observed entries. Consider $r+1$ of the observed entries of the first row of $\mathbf{U}_{(d)}$ and we denote them by $\mathcal{U}(\vec{x}_{1})$, $\dots$, $\mathcal{U}(\vec{x}_{r+1})$, where the last component of the vector $\vec{x}_i$ is equal to one, $1 \leq i \leq r+1$. Similar to the proof of Theorem \ref{thmmat2}, genericity of $\mathcal{U}$ results in 
\begin{eqnarray}\label{polyCPtensolindep}
\mathcal{U}(\vec{x}_{r+1}) = \sum_{l=1}^{r} t_l \mathcal{U}(\vec{x}_{i}),
\end{eqnarray}
where $t_l$'s are constant scalars, $l=1,\dots,r$. On the other hand, according to Lemma \ref{thmten1CP} there exist at most finitely many completions of the sampled tensor of rank $r$. Therefore, there exist a completion of $\mathbf{U}$ of rank $r$ with probability zero. Moreover, an equation of the form of \eqref{polyCPtensolindep} holds with probability zero as $r^* \geq r+1$ and $\mathcal{U}$ is chosen generically from the manifold of tensors of rank-$r^*$. Therefore, there exists a completion of rank $r$ with probability zero.
\end{proof}

\begin{corollary}\label{colthmtensorCP1}
Consider an arbitrary number $r^{\prime} \in \mathcal{S}_{{\Omega}}$. Similar to Theorem \ref{thmtenCPr2}, it follows that with probability one, exactly one of the followings holds

(i) $r^* \in  \{1,2,\dots,r^{\prime} \}${\rm ;}

(ii) For any arbitrary completion of the sampled tensor $\mathcal{U}$ of rank $r$, we have $r \notin \{1,2,\dots,r^{\prime} \}$.
\end{corollary}

\begin{corollary}\label{colthmtensorCP1nw}
Assuming that there exists a CP rank-$r$ completion of the sampled tensor $\mathcal{U}$ such that $r \in \mathcal{S}_{{\Omega}}$, we conclude that with probability one $r^* \leq r$.
\end{corollary}

\begin{corollary}\label{colthmtensorCPopt1}
Let $\mathcal{U}^*$ denote an optimal solution to the following NP-hard optimization problem
{\rm
\begin{align}\label{optprotens}
& \ \ \ \ \ \  \ \ \ \ \ \  \   \ \text{minimize}_{\mathcal{U}^{\prime} \in \mathbb{R}^{n_1  \times \cdots \times n_d}}
& &  \text{rank}_{\text{CP}}(\mathcal{U}^{\prime}) \ \ \ \ \  \ \ \  \ \ \  \ \ \  \ \ \  \ \ \  \ \ \  \ \ \  \ \ \   \\
& \ \ \ \ \ \  \ \ \ \ \ \  \   \ \text{subject to}
& & \mathcal{U}^{\prime}_{{\Omega}} = \mathcal{U}_{{\Omega}}. \nonumber
\end{align}
}
Assume that {\rm $\text{rank}_{\text{CP}}(\mathcal{U}^*) \in \mathcal{S}_{{\Omega}}$}. Then, Corollary \ref{colthmtensorCP1nw} results that {\rm $r^* = \text{rank}_{\text{CP}}(\mathcal{U}^*)$} with probability one.
\end{corollary}

The following lemma is Lemma $15$ in \cite{ashraphijuo4}, which is the probabilistic version of Lemma \ref{thmten1CP} in terms of the sampling probability.

\begin{lemma}\label{prsCPfinRD}
Assume that $n_1=n_2=\dots = n_d=n$, $d>2$, $ n > \max \{ 200, r(d-2)\}$ and $r \leq \frac{n}{6}$. Moreover, assume that the sampling probability satisfies
\begin{eqnarray}\label{prsCPbdfinRD}
p > \frac{1}{n^{d-2}} \max\left\{27 \ \log \left( \frac{n}{\epsilon} \right) + 9 \ \log \left( \frac{2r(d-2)}{\epsilon} \right) + 18, 6r \right\} + \frac{1}{\sqrt[4]{n^{d-2}}}.
\end{eqnarray}
Then, with probability at least $ (1- \epsilon) \left( 1-\exp(-\frac{\sqrt{n^{d-2}}}{2}) \right)^{n^2}$, we have $r \in \mathcal{S}_{\Omega}$.
\end{lemma}

The following corollary is the probabilistic version of Corollaries \ref{colthmtensorCP1nw} and \ref{colthmtensorCPopt1}.

\begin{corollary}\label{colthmtensorCP1nwtq}
Assuming that there exists a CP rank-$r$ completion of the sampled tensor $\mathcal{U}$ such that the conditions given in Lemma \ref{prsCPfinRD} hold, with the sampling probability satisfying \eqref{prsCPbdfinRD}, we conclude that with probability at least $ (1- \epsilon) \left( 1-\exp(-\frac{\sqrt{n^{d-2}}}{2}) \right)^{n^2}$ we have $r^* \leq r$. Therefore, given that \eqref{prsCPbdfinRD} holds for {\rm $r=\text{rank}(\mathbf{U}^*)$} and $\mathbf{U}^*$ denotes an optimal solution to the optimization problem \eqref{optprotens}, with probability at least $ (1- \epsilon) \left( 1-\exp(-\frac{\sqrt{n^{d-2}}}{2}) \right)^{n^2}$ we have {\rm $r^* = \text{rank}(\mathbf{U}^*)$}.
\end{corollary}


\section{Vector-Rank Cases}\label{rankvecmo}

\subsection{Multi-View Matrix}\label{multiview}

 The following assumptions will be used frequently in this subsection.

{\bf Assumption $A_{r_1,r_2}$}: Each column of $\mathbf{U}_1$ and $\mathbf{U}_2$ include at least $r_1$ and $r_2$ sampled entries, respectively.

We construct a binary valued matrix called {\bf constraint matrix} for multi-view matrix $\mathbf{U}$ as $\mathbf{\breve{\Omega}}_{r_1,r_2}=[\mathbf{\breve{\Omega}}_{r_1}|\mathbf{\breve{\Omega}}_{r_2}]$, where $\mathbf{\breve{\Omega}}_{r_1}$ and $\mathbf{\breve{\Omega}}_{r_2}$ represent the constraint matrix for single-view matrices $\mathbf{U}_1$ and $\mathbf{U}_2$ (defined in Section \ref{svmsubs}), respectively.

{\bf Assumption $B_{r_1,r_2,r}$}: $\mathbf{\breve{\Omega}}_{r_1,r_2}$ consists a submatrix $\mathbf{\breve{\Omega}}_{r_1,r_2}^{\prime} \in \mathbb{R}^{n \times K}$ such that $K= nr - r^2 - r_1^2 - r_2^2 + r(r_1+r_2)$ and for any $K^{\prime} \in \{1,2,\dots,K\}$ and any submatrix $\mathbf{\breve{\Omega}}_{r_1,r_2}^{\prime \prime} \in \mathbb{R}^{n \times K^{\prime}}$ of $\mathbf{\breve{\Omega}}_{r_1,r_2}^{\prime}$ we have
\begin{eqnarray}\label{matfineqdanmv}
(r-r_2) \left( f(\mathbf{\breve{\Omega}}_{r_1}^{\prime \prime}) - r_1 \right)^+ + (r-r_1) \left( f(\mathbf{\breve{\Omega}}_{r_2}^{\prime \prime}) - r_2 \right)^+ \nonumber \\ + (r_1 + r_2 -r) \left( f(\mathbf{\breve{\Omega}}_{r_1,r_2}^{\prime \prime}) - (r_1 + r_2 -r) \right)^+ \geq K^{\prime},
\end{eqnarray}
where $f(\mathbf{X})$ denotes the number of nonzero rows of $\mathbf{X}$ for any matrix $\mathbf{X}$ and $\mathbf{\breve{\Omega}}_{r_1,r_2}^{\prime \prime}= [\mathbf{\breve{\Omega}}_{r_1}^{\prime \prime}|\mathbf{\breve{\Omega}}_{r_2}^{\prime \prime}]$, and also $\mathbf{\breve{\Omega}}_{r_1}^{\prime \prime}$ and $\mathbf{\breve{\Omega}}_{r_2}^{\prime \prime}$ denote the columns of $\mathbf{\breve{\Omega}}_{r_1,r_2}^{\prime \prime}$ corresponding to $\mathbf{\breve{\Omega}}_{r_1}$ and $\mathbf{\breve{\Omega}}_{r_2}$, respectively.

The following lemma is a re-statement of Theorem $2$ in \cite{ashraphijuo2}. 

\begin{lemma}\label{thmmat1mv}
For almost every $\mathbf{U}$, there are only finitely many completions of the sampled multi-view data if and only if Assumptions $A_{r_1^*,r_2^*}$ and $B_{r_1^*,r_2^*,r^*}$ hold.
\end{lemma}


\begin{definition}
Denote the rank vector $\underline{ r}=(r_1,r_2,r)$. Define the generalized inequality $\underline{ r}^{\prime} \preceq  \underline{ r}$ as the component-wise set of inequalities, e.g., $r_1^{\prime} \leq r_1$, $r_2^{\prime} \leq r_2$ and $r^{\prime} \leq r$.
\end{definition}

\begin{definition}
Let $\mathcal{S}_{\mathbf{\Omega}}$ denote the set of all $\underline{r}$ such that both Assumptions $A_{r_1,r_2}$ and $B_{r_1,r_2,r}$ hold.
\end{definition}


\begin{lemma}\label{multivminimorder}
Assume $\underline{ r} \in  \mathcal{S}_{\mathbf{\Omega}} $. Then, for any  $\underline{ r}^{\prime} \preceq  \underline{ r}$, we have $\underline{ r}^{\prime} \in  \mathcal{S}_{\mathbf{\Omega}} $.
\end{lemma}

\begin{proof}
We consider the rank factorization of $\mathbf{U}$ as in \cite{ashraphijuo2} and similar to the single-view scenario in Lemma \ref{sinvorderkam} each observed entry results in a polynomial in terms of the entries of the components of the decomposition. Note that the dimension of the manifold corresponding to rank vector $\underline{ r}$ is equal to $rn+r_1n_1+r_2n_2 - r^2 - r_1^2 - r_2^2 + r(r_1+r_2)$ \cite{ashraphijuo2}, and also observe that the fact that $\max \{r_1,r_2\} \leq r \leq r_1 + r_2  \leq \min \{2n,n_1+n_2\}$ implies that reducing any of the values $r_1,r_2$, and $r$ reduces the value of $rn+r_1n_1+r_2n_2 - r^2 - r_1^2 - r_2^2 + r(r_1+r_2)$. Hence, the dimension of the manifold corresponding to  rank vector $\underline{ r}$ is larger than that for rank vector $\underline{ r}^{\prime}$, given $\underline{ r}^{\prime} \preceq  \underline{ r}$, and thus similar to the proof of Lemma \ref{sinvorderkam}, finite completability of data with $\underline{ r}$  results finite completability of data with  $\underline{ r}^{\prime}$ with probability one. Then, using Lemma \ref{thmmat1mv}, the proof is complete.
\end{proof}

The following theorem provides a relationship between the unknown rank vector $\underline{ r}^*$ and $\mathcal{S}_{\mathbf{\Omega}}$.

\begin{theorem}\label{thmmat2multivw}
For almost every $\mathbf{U}$, with probability one, exactly one of the following statements holds

(i) $\underline{ r}^* \in \mathcal{S}_{\mathbf{\Omega}}${\rm ;}

(ii) For any arbitrary completion of the sampled matrix $\mathbf{U}$ of rank vector  $\underline{ r}$, we have $\underline{ r} \notin \mathcal{S}_{\mathbf{\Omega}}$.
\end{theorem}

\begin{proof}
Similar to the proof of Theorem \ref{thmmat2}, suppose that there does not exist a completion of $\mathbf{U}$ of rank vector $\underline{ r}$ such that  $\underline{ r} \in \mathcal{S}_{\mathbf{\Omega}}$. Therefore, it is easily verified that statement (ii) holds and statement (i) does not hold. On the other hand, assume that there exists a completion of $\mathbf{U}$ of rank vector $\underline{r}$, where  $\underline{ r} \in \mathcal{S}_{\mathbf{\Omega}}$. Hence, statement (ii) does not hold and to complete the proof it suffices to show that with probability one, statement (i) holds. Similar to Theorem \ref{thmmat2}, we show that assuming $\underline{ r}^* \notin \mathcal{S}_{\mathbf{\Omega}}$, there exists a completion of $\mathbf{U}$ of rank vector $\underline{ r}$, where  $\underline{ r} \in \mathcal{S}_{\mathbf{\Omega}}$, with probability zero.


Since $\underline{ r}^* \notin \mathcal{S}_{\mathbf{\Omega}}$, according to Lemma \ref{multivminimorder}, for any $\underline{ r} \in \mathcal{S}_{\mathbf{\Omega}}$ at least one the following inequalities holds; $r_1 < r_1^*$, $r_2 < r_2^*$ and $r < r^*$. Note that assuming that there exists a completion of $\mathbf{U}_1$ of rank $r_1$ with probability zero results that there exists a completion of $\mathbf{U}$ of rank vector $\underline{ r}$ with probability zero and similar statement holds for $r_2$ and $r$. Hence, in any possible scenario ($r_1 < r_1^*$ or $r_2 < r_2^*$ or $r < r^*$) the similar proof as in Theorem \ref{thmmat2} (for single-view matrix) results that there exists a completion of $\mathbf{U}$ of rank vector $\underline{ r}$, where  $\underline{ r} \in \mathcal{S}_{\mathbf{\Omega}}$, with probability zero.
\end{proof}

\begin{corollary}\label{colthmmatmv1}
Consider a subset $\mathcal{S}_{\mathbf{\Omega}}^{\prime}$ of $\mathcal{S}_{\mathbf{\Omega}}$ such that for any two members of $\mathcal{S}_{\mathbf{\Omega}}$ that $\underline{ r}^{\prime} \preceq \underline{ r}^{\prime \prime}$ and $\underline{ r}^{\prime \prime} \in \mathcal{S}_{\mathbf{\Omega}}^{\prime}$ we have $\underline{ r}^{\prime } \in \mathcal{S}_{\mathbf{\Omega}}^{\prime}$. Then, with probability one, exactly one of the followings holds

(i) $\underline{ r}^* \in \mathcal{S}_{\mathbf{\Omega}}^{\prime} ${\rm ;}

(ii) For any arbitrary completion of $\mathbf{U}$ of rank vector $\underline{ r}$, we have $\underline{ r} \notin \mathcal{S}_{\mathbf{\Omega}}^{\prime} $.
\end{corollary}

\begin{proof}
Note that the property in the statement of Lemma \ref{multivminimorder} holds for $\mathcal{S}_{\mathbf{\Omega}}^{\prime} $ as well as $\mathcal{S}_{\mathbf{\Omega}} $. Moreover, as $\mathcal{S}_{\mathbf{\Omega}}^{\prime} \subseteq \mathcal{S}_{\mathbf{\Omega}}$, for any $\underline{ r} \in \mathcal{S}_{\mathbf{\Omega}}^{\prime} $ there exists at most finitely many completions of $\mathbf{U}$ of rank vector $\underline{ r}$, and therefore the rest of the proof is the same as the proof of  Theorem \ref{thmmat2multivw}.
\end{proof}

\begin{corollary}\label{colthmmatmv1nw}
Assuming that there exists a completion of $\mathbf{U}$ with rank vector $\underline{ r}$ such that $\underline{ r}\in \mathcal{S}_{\mathbf{\Omega}}$, then with probability one $\underline{ r}^*  \preceq \underline{ r} $.
\end{corollary}

The following lemma which is a re-statement of Theorem $3$ in \cite{ashraphijuo2} gives the number of samples per column that is needed to ensure that Assumptions $A_{r_1,r_2}$ and $B_{r_1,r_2,r}$ hold with high probability.

\begin{lemma}\label{danthmmultviwmat}
Suppose that the following inequalities hold 
\begin{eqnarray}
\frac{n}{6} &\geq & \max\{r_1,r_2,(r_1+r_2-r)\}, \label{as1} \\ 
n_1 &\geq & (r-r_2)(n-r_1), \label{as2} \\
n_2 &\geq & (r-r_1)(n-r_2), \label{as3} \\
n_1+n_2 &\geq & (r-r_2)(n-r_1) + (r-r_1)(n-r_2) \nonumber \\ &+ &   (r_1+r_2-r)(n-(r_1+r_2-r)). \label{as4}
\end{eqnarray}
Moreover assume that each {\bf column} of $\mathbf{U}$ is observed in at least $l$ entries, uniformly at random and independently across entries, where
\begin{eqnarray}\label{genmatrixmulti}
l > \max \left\{9 \ \log \left( \frac{n}{\epsilon} \right) + 3 \ \log \left( \frac{3 \max \left\{ r-r_2,r-r_1,r_1+r_2-r \right\} } {\epsilon}\right) + 6, 2r_1, 2r_2\right\}.
\end{eqnarray}
Then, with probability at least $1 - \epsilon$, $\underline{ r} \in \mathcal{S}_{\mathbf{\Omega}}$.
\end{lemma}

The following proposition is the probabilistic version of Theorem \ref{thmmat2multivw} in terms of the sampling probability instead of verifying Assumptions $A_{r_1,r_2}$ and $B_{r_1,r_2,r}$.

\begin{proposition}\label{thmprobsvmsampprjw}
Suppose that \eqref{as1}-\eqref{as4} hold for $\underline{ r} $ and that each entry of the sampled matrix is observed  uniformly at random and independently across entries with probability $p$, where
\begin{eqnarray}\label{genmatrixrsampprojw}
p > \frac{1}{n} \max \left\{9 \ \log \left( \frac{n}{\epsilon} \right) + 3 \ \log \left( \frac{3 \max \left\{ r-r_2,r-r_1,r_1+r_2-r \right\} } {\epsilon}\right) + 6, 2r_1, 2r_2\right\} + \frac{1}{\sqrt[4]{n}}. 
\end{eqnarray}
Then, with probability at least $\left( 1 - \epsilon \right)\left(1-\exp(-\frac{\sqrt{n}}{2})\right)^{n_1+n_2}$, we have $\underline{ r} \in \mathcal{S}_{\mathbf{\Omega}}$. 
\end{proposition}

\begin{proof}
The proposition is easy to verify using Lemma \ref{danthmmultviwmat} and Lemma \ref{danthmsingviwmat} (similar to the proof for Proposition \ref{thmprobsvmsamppr}).
\end{proof}





\begin{corollary}\label{colthmtCMv1nwtq}
Assuming that there exists a completion of $\mathbf{U}$ of rank vector $\underline{ r}$ such that \eqref{as1}-\eqref{as4} hold and the sampling probability satisfies \eqref{genmatrixrsampprojw}, then with probability at least $\left( 1 - \epsilon \right)\left(1-\exp(-\frac{\sqrt{n}}{2})\right)^{n_1+n_2}$ we have $\underline{ r}^* \preceq \underline{ r}$.
\end{corollary}




\subsection{Tucker-Rank Tensor}\label{Tucker}

Assume that the sampled tensor $\mathcal{U}$ is chosen generically from the manifold of tensors of rank $\underline{ r}^* = \text{rank}_{\text{Tucker}} (\mathcal{U})=(m_{j+1}^*,\dots,m_{d}^*)$, where $m_i^*= \text{rank}(\mathbf{U}_{(i)})$, i.e., the rank of the $i$-th matricization of $\mathcal{U}$, is unknown, $j+1 \leq i \leq d$ and $j \in \{1,\dots,d-1\}$.



Without loss of generality assume that $m_{j+1}^* \geq \dots \geq m_d^*$ throughout this subsection. Also, given $\underline{ r} = (m_{j+1},\dots,m_d)$, define the following function
\begin{eqnarray}\label{defuntu}
g_{\underline{ r}}(x) = \sum_{i=j+1}^{d}  \min \left\{  r_i ,\left( x - \sum_{i^{\prime}=j+1}^{i-1} r_{i^{\prime}} \right)^+ \right\} r_i.
\end{eqnarray}

\begin{definition}
For any $i \in \{j+1,\dots,d\}$ and $ \mathcal{S}_i \subseteq \{1,\dots,n_i\}$, define $\mathcal{U}^{(\mathcal{S}_{i})}$ as a set containing the entries of $|\mathcal{S}_i|$ rows (corresponding to the elements of $\mathcal{S}_i$) of $\mathbf{U}_{(i)}$. Moreover, define $\mathcal{U}^{(\mathcal{S}_{j+1},\dots,\mathcal{S}_{d})} = \mathcal{U}^{(\mathcal{S}_{j+1})}\cup \dots \cup \mathcal{U}^{(\mathcal{S}_{d})}$. 
\end{definition}

{\bf Assumption }$\mathcal{A}_{\underline{ r}}^{\text{Tucker}}$: There exist $\sum_{i=j+1}^{d} \left( n_im_i \right)$ observed entries such that for any $ \mathcal{S}_i \subseteq \{1,\dots,n_i\}$ for $i \in \{j+1,\dots,d\}$, $\mathcal{U}^{(\mathcal{S}_{j+1},\dots,\mathcal{S}_{d})} $ includes at most $\sum_{i=j+1}^{d} |\mathcal{S}_i|m_i$ of the mentioned $\sum_{i=j+1}^{d} \left( n_im_i \right)$ observed entries.

Let $\mathcal{P}$ be a set of $\sum_{i=j+1}^{d} \left( n_im_i \right)$ observed entries such that they satisfy Assumption $\mathcal{A}_{\underline{ r}}^{\text{Tucker}}$. Now, we construct a $(j+1)^{\text{th}}$-order binary constraint tensor ${\breve{\Omega}}_{\underline{ r}}$ in some sense similar to that in Section \ref{CPtensubs}. For any subtensor $\mathcal{Y} \in \mathbb{R}^{n_1 \times n_2 \times \cdots \times n_j \times 1  \times \cdots \times 1 }$ of the tensor $\mathcal{U}$, let $N_{\Omega}(\mathcal{Y}^{\mathcal{P}})$ denote the number of sampled entries in $\mathcal{Y}$ that belong to $\mathcal{P}$.

The sampled tensor $\mathcal{U}$ includes $n_{j+1}  n_{j+2} \cdots n_d$ subtensors that belong to  $\mathbb{R}^{n_1 \times n_2 \times \cdots \times n_j \times 1  \times \cdots \times 1 }$ and we label these subtensors by $\mathcal{Y}_{(t_{j+1},\dots,t_{d})}$ where ${(t_{j+1},\dots,t_{d})}$ represents the coordinate of the subtensor. Define a binary valued tensor $\mathcal{\breve{Y}}_{(t_{j+1},\cdots,t_{d})} \in \mathbb{R}^{n_1 \times n_2 \times \cdots \times n_j \times \overbrace{ 1  \times \dots \times 1}^{d-j} \times  k}$, where $k= N_{\Omega}(\mathcal{Y}_{(t_{j+1},\ldots,t_{d})}) - N_{\Omega}(\mathcal{Y}_{(t_{j+1},\ldots,t_{d})}^{\mathcal{P}})$ and its entries are described as the following. We can look at $\mathcal{\breve{Y}}_{(t_{j+1},\cdots,t_{d})}$ as $k$ tensors each belongs to $\mathbb{R}^{n_1 \times n_2 \times \cdots \times n_j \times 1  \times \cdots \times 1 }$. For each of the mentioned $k$ tensors in $\mathcal{\breve{Y}}_{(t_{j+1},\cdots,t_{d})}$ we set the entries corresponding to the $N_{\Omega}(\mathcal{Y}_{(t_{j+1},\ldots,t_{d})}^{\mathcal{P}})$ observed entries that belong to $\mathcal{P}$ equal to $1$. For each of the other $k$ observed entries, we pick one of the $k$ tensors of $\mathcal{\breve{Y}}_{(t_{j+1},\cdots,t_{d})}$ and set its corresponding entry (the same location as that specific observed entry) equal to $1$ and set the rest of the entries equal to $0$.

For the sake of simplicity in notation, we treat tensors $\mathcal{\breve{Y}}_{(t_{j+1},\cdots,t_{d})}$ as a member of $\mathbb{R}^{n_1 \times n_2 \times \cdots \times n_j \times  k}$ instead of  $\mathbb{R}^{n_1 \times n_2 \times \cdots \times n_j \times \overbrace{ 1  \times \cdots \times 1}^{d-j} \times  k}$. Now, by putting together all $n_{j+1}  n_{j+2} \cdots n_d$ tensors in dimension $(j+1)$, we construct a binary valued tensor ${\breve{\Omega}}_{\underline{ r}} \in \mathbb{R}^{n_1 \times n_2 \times \cdots \times n_j \times K_j}$, where $K_j = N_{\Omega}(\mathcal{U}) - \sum_{i=j+1}^{d} \left( n_im_i\right)$ and call it the {\bf constraint tensor} \cite{ashraphijuo}. In \cite{ashraphijuo}, an example is given on the construction of ${\breve{\Omega}}_{\underline{r}}$.

{\bf Assumption }$\mathcal{B}_{\underline{ r}}^{\text{Tucker}}$: The constraint tensor ${\breve{\Omega}}_{\underline{ r}}$ consists a subtensor ${\breve{\Omega}}_{\underline{ r}}^{\prime} \in \mathbb{R}^{n_1 \times n_2 \times \cdots \times n_j \times K}$ such that $K= \left( \Pi_{i=1}^{j} n_i \right) \left( \Pi_{i=j+1}^{d} m_i \right) - \sum_{i=j+1}^{d} m_i^2 $ and for any $K^{\prime} \in \{1,2,\dots,K\}$ and any subtensor ${\breve{\Omega}}_{\underline{ r}}^{\prime \prime} \in \mathbb{R}^{n_1 \times n_2 \times \cdots \times n_{d-1} \times K^{\prime}}$ of ${\breve{\Omega}}_{\underline{ r}}^{\prime}$ we have
\begin{eqnarray}\label{matfineqdanmvTucker}
\left( \Pi_{i=j+1}^{d} m_i \right) \left(  f_{j+1} (\breve{\Omega}^{\prime \prime}_{\underline{ r}})  \right) - g_{\underline{ r}} \left(  f_{j+1} (\breve{\Omega}^{\prime \prime}_{\underline{ r}})  \right) \geq K^{\prime},
\end{eqnarray}
where $f_{j+1} (\breve{\Omega}^{\prime \prime}_{\underline{ r}})$ denotes the number of nonzero columns of the $(j+1)$-th matricization of $\breve{\Omega}^{\prime \prime}_{\underline{ r}}$.

The following lemma is a re-statement of Theorem $3$ in \cite{ashraphijuo}.

\begin{lemma}\label{thmten1Tucker}
For almost every $\mathcal{U}$, there are only finitely many completions of rank ${\underline{ r}^*}$ of the sampled tensor if and only if Assumptions $\mathcal{A}_{\underline{ r}^*}^{\text{Tucker}}$ and $\mathcal{B}_{\underline{ r}^*}^{\text{Tucker}}$ hold.
\end{lemma}

\begin{definition}
Let $\mathcal{S}_{{\Omega}}$ denote the set of all rank vectors $\underline{ r}$ such that both Assumptions $\mathcal{A}_{\underline{ r}}^{\text{Tucker}}$ and $\mathcal{B}_{\underline{ r}}^{\text{Tucker}}$ hold.
\end{definition}


\begin{lemma}\label{tensorTuckerminimorder}
Assume $\underline{ r} \in  \mathcal{S}_{{\Omega}} $. Then, for any rank vector $\underline{ r}^{\prime} \preceq \underline{ r}$, we have $\underline{ r}^{\prime} \in  \mathcal{S}_{{\Omega}} $.
\end{lemma}

\begin{proof}
Note that the dimension of the manifold corresponding to $\underline{ r}$ is $\left( \Pi_{i=1}^{j} n_i \right) \left( \Pi_{i=j+1}^{d} m_i \right) + \sum_{i=j+1}^{d} n_im_i- \sum_{i=j+1}^{d} m_i^2$ \cite{ashraphijuo}, and thus by reducing the value of $m_{i_0}$ by one (for $i_0 \in \{j+1,\dots,d\}$), the value of the mentioned dimension reduces by at least $\left( \Pi_{i=1}^{j} n_i \right) + n_i - 2m_i + 1$, which is greater than zero since $m_i \leq n_i$. The rest of the proof is similar to the proof of Lemma \ref{sinvorderkam}.
\end{proof}

\begin{definition}
Define $\mathcal{S}_{\Omega}(\underline{ r})$ as a subset of $\mathcal{S}_{\Omega}$, which includes all  $\underline{ r}^{\prime} \in \mathcal{S}_{\Omega}$ that  $\underline{ r}^{\prime}  \preceq  \underline{ r}$.
\end{definition}

The following theorem gives a relationship between $\underline{r}^*$ and ${\mathcal{S}}_{{\Omega}}$.

\begin{theorem}\label{thmmat2tensorTucker}
For almost every $\mathcal{U}$, with probability one, exactly one of the following statements holds

(i) $\underline{ r}^*  \in {\mathcal{S}}_{{\Omega}}${\rm ;}

(ii) For any arbitrary completion of the sampled tensor $\mathcal{U}$ of rank  $\underline{ r}$, we have $\underline{ r} \notin {\mathcal{S}}_{{\Omega}}(\underline{ r}^*)$.
\end{theorem}

\begin{proof}
Similar to the proof of Theorem \ref{thmmat2}, to complete the proof it suffices to show that the assumption $\underline{ r}^* \notin {\mathcal{S}}_{{\Omega}}$ results that there exists a completion of $\mathcal{U}$ of rank $\underline{ r}$, where  $\underline{ r} \in {\mathcal{S}}_{{\Omega}}(\underline{ r}^*)$, with probability zero. Note that $\underline{ r} \in {\mathcal{S}}_{{\Omega}}(\underline{ r}^*) \subseteq \mathcal{S}_{\Omega}$ results that Assumptions $\mathcal{A}_{\underline{ r}}^{\text{Tucker}}$ and $\mathcal{B}_{\underline{ r}}^{\text{Tucker}}$ hold. Moreover, note that $\underline{ r}  \preceq  \underline{ r}^*$ and since $\underline{ r}^* \notin {\mathcal{S}}_{{\Omega}}$ we conclude that there exists $i_0 \in \{j+1,\dots,d\}$ such that $m_{i_0} < m_{i_0}^*$. As a result, $\sum_{i=j+1}^{d} n_im_i < \sum_{i=j+1}^{d} n_im_i^* $. 

Assumption $\mathcal{B}_{\underline{ r}}^{\text{Tucker}}$ ensures there exists at least one more observed entry (otherwise the constraint tensor does not exist) besides the $\sum_{i=j+1}^{d} n_im_i$ mentioned observed entries. Given the basis $\mathcal{C} \in \mathbb{R}^{n_1 \times \dots \times n_j \times m_{j+1} \times \dots \times m_d}$ as in \eqref{TuckTuck}, there exist $\sum_{i=j+1}^{d} n_im_i$ variables in the corresponding Tucker decomposition. However, we have $\sum_{i=j+1}^{d} n_im_i + 1$ polynomials in terms these $\sum_{i=j+1}^{d} n_im_i$ variables and therefore the last polynomials can be written as algebraic combination of the other $\sum_{i=j+1}^{d} n_im_i$ polynomials. This leads to a linear equation in terms of the $\sum_{i=j+1}^{d} n_im_i + 1$ corresponding observed entries. On the other hand, the $\sum_{i=j+1}^{d} n_im_i$ observed entries satisfy the property stated as Assumption $\mathcal{A}_{\underline{ r}}^{\text{Tucker}}$ and it is easily verified that there exist $\sum_{i=j+1}^{d} n_im_i^* $ entries (observed and non-observed) satisfying Assumption $\mathcal{A}_{\underline{ r}^*}^{\text{Tucker}}$ such that the union of the mentioned $\sum_{i=j+1}^{d} n_im_i$ entries with any arbitrary other observed entry be a subset of those $\sum_{i=j+1}^{d} n_im_i^* $ entries. However, $\mathcal{U}$ is generically chosen from the manifold corresponding to $\underline{ r}^*$ and therefore a particular linear equation in terms of the mentioned $\sum_{i=j+1}^{d} n_im_i^* $ entries holds with probability zero. The rest of the proof is similar to the proof of Theorem \ref{thmmat2}.
\end{proof}

\begin{corollary}\label{colthmtensorTucker1nw}
Assuming that there exists a completion of $\mathcal{U}$ with rank vector $\underline{ r}$ such that $\underline{ r} \in  {\mathcal{S}}_{\Omega}$, we conclude that with probability one $\underline{ r}^* \preceq \underline{ r} $. 
\end{corollary}

The following lemma is Corollary $2$ in \cite{ashraphijuo}, which ensures that Assumptions $\mathcal{A}_{\underline{ r}}^{\text{Tucker}}$ and $\mathcal{B}_{\underline{ r}}^{\text{Tucker}}$ hold with high probability.

\begin{lemma}\label{lemghTuck}
Assume that $\sum_{i=j+1}^{d} m_i^2 \leq \Pi_{i=j+1}^{d} m_i$,  $\Pi_{i=j+1}^{d}n_i \geq N_j  \Pi_{i=j+1}^{d} m_i - \sum_{i=j+1}^{d} m_i^2$, $ \Pi_{i=j+1}^{d} m_i \leq N_j$, where $N_j = \Pi_{i=1}^{j} n_i $. Furthermore, assume that we observe each entry of  $\mathcal{U}$ with probability  $p$, where
\begin{eqnarray}\label{coin2Tuck}
p >  \frac{1}{N_j} \left( 6 \ \log \left( {N_j} \right) + 2 \log \left(  \max \left\{ \frac{2 \sum_{i=j+1}^{d} r_i^2}{\epsilon} ,   \frac{2\Pi_{i=j+1}^{d} r_i - 2\sum_{i=j+1}^{d} r_i^2}{\epsilon} \right\} \right) + 4  \right)+ \frac{1}{\sqrt[4]{N_j}}.
\end{eqnarray}
Then, with probability  at least $(1- \epsilon) \left( 1-\exp(-\frac{\sqrt{\Pi_{i=1}^{j} n_i}}{2}) \right)^{\Pi_{i=j+1}^{d} n_i}$, $\underline{ r} \in \mathcal{S}_{\Omega}$.
\end{lemma}

The following corollary is the probabilistic version of Theorem \ref{thmmat2tensorTucker}.

\begin{corollary}\label{colthmtensorTuk1nwtq}
Assuming that there exists a completion of the sampled tensor $\mathcal{U}$ of Tucker rank $\underline{ r}$ such that the assumptions in Lemma \ref{lemghTuck} hold and the sampling probability satisfies \eqref{coin2Tuck}, then with probability at least $(1- \epsilon) \left( 1-\exp(-\frac{\sqrt{\Pi_{i=1}^{j} n_i}}{2}) \right)^{\Pi_{i=j+1}^{d} n_i}$ we have $\underline{ r}^* \preceq \underline{ r}$.
\end{corollary}

\subsection{TT-Rank Tensor}\label{TT}

Assume that the sampled tensor $\mathcal{U}$ is chosen generically from the manifold of tensors of rank $\underline{ r}^*= \text{rank}_{\text{TT}} (\mathcal{U})=(u_1^*,\dots,u_{d-1}^*)$, where $u_i^*=\text{rank}(\mathbf{\widetilde U}_{(i)})$, i.e., the rank of the $i$-th unfolding of $\mathcal{U}$, is unknown, $1 \leq i \leq d-1$. Define $u_0 = u_d =1$.


{\bf Assumption }$\mathcal{A}_{\underline{ r}}^{\text{TT}}$: Each row of the $d$-th matricization of the sampled tensor, i.e., $\mathbf{U}_{(d)}$ includes at least $u_{d-1}$ observed entries.

We construct the $d$-way binary valued constraint tensor ${\breve{\Omega}}_{u_{d-1}}$ similar to that in Section \ref{CPtensubs} as the following. Consider any subtensor $\mathcal{Y} \in \mathbb{R}^{n_1 \times n_2 \times \cdots \times n_{d-1} \times 1 }$ of the tensor $\mathcal{U}$. The sampled tensor $\mathcal{U}$ includes $n_d$ subtensors that belong to  $\mathbb{R}^{n_1 \times n_2 \times \cdots \times n_{d-1} \times 1 }$ and let $\mathcal{Y}_i$ for $1 \leq i \leq n_d$ denote these $n_d$ subtensors. Define a binary valued tensor $\mathcal{\breve{Y}}_{i} \in \mathbb{R}^{n_1 \times n_2 \times \cdots \times n_{d-1} \times  k_i}$, where $k_i= N_{\Omega}(\mathcal{Y}_{i}) - u_{d-1}$ and its entries are described as the following. We can look at $\mathcal{\breve{Y}}_{i}$ as $k_i$ tensors each belongs to $\mathbb{R}^{n_1 \times n_2 \times \cdots \times n_{d-1} \times 1 }$. For each of the mentioned $k_i$ tensors in $\mathcal{\breve{Y}}_{i}$ we set the entries corresponding to $u_{d-1}$ of the observed entries equal to $1$. For each of the other $k_i$ observed entries, we pick one of the $k_i$ tensors of $\mathcal{\breve{Y}}_{i}$ and set its corresponding entry (the same location as that specific observed entry) equal to $1$ and set the rest of the entries equal to $0$. In the case that $k_i=0$ we simply ignore $\mathcal{\breve{Y}}_{i}$, i.e., $\mathcal{\breve{Y}}_{i} = \emptyset$

By putting together all $n_d$ tensors in dimension $d$, we construct a binary valued tensor ${\breve{\Omega}}_{u_{d-1}} \in \mathbb{R}^{n_1 \times n_2 \times \cdots \times n_{d-1} \times K}$, where $K = \sum_{i=1}^{n_d} k_i = N_{\Omega}(\mathcal{U}) - u_{d-1} n_d$ and call it the {\bf constraint tensor}. Observe that each subtensor of ${\breve{\Omega}}_{u_{d-1}}$ which belongs to $\mathbb{R}^{n_1 \times n_2 \times \cdots \times n_{d-1} \times 1 }$ includes exactly $u_{d-1}+1$ nonzero entries. In \cite{ashraphijuo3}, an example is given on the construction of ${\breve{\Omega}}_{u_{d-1}}$.

{\bf Assumption }$\mathcal{B}_{\underline{ r}}^{\text{TT}}$: ${\breve{\Omega}}_{u_{d-1}}$ consists a subtensor ${\breve{\Omega}}_{u_{d-1}}^{\prime} \in \mathbb{R}^{n_1 \times n_2 \times \cdots \times n_{d-1} \times K}$ such that $K= \sum_{i=1}^{d-1} u_{i-1}n_iu_i - \sum_{i=1}^{d-1} u_i^2$ and for any $K^{\prime} \in \{1,2,\dots,K\}$ and any subtensor ${\breve{\Omega}}_{u_{d-1}}^{\prime \prime} \in \mathbb{R}^{n_1 \times n_2 \times \cdots \times n_{d-1} \times K^{\prime}}$ of ${\breve{\Omega}}_{u_{d-1}}^{\prime}$ we have
\begin{eqnarray}\label{matfineqdanmvTT}
\sum_{i=1}^{d-1} \left(  u_{i-1}f_i (\breve{\Omega}^{\prime \prime}_{u_{d-1}}) u_i - u_i^2 \right)^+  \geq K^{\prime},
\end{eqnarray}
where $f_i (\breve{\Omega}^{\prime \prime}_{u_{d-1}})$ denotes the number of nonzero rows of the $i$-th matricization of $\breve{\Omega}^{\prime \prime}_{u_{d-1}}$.

The following lemma is a re-statement of Theorem $1$ in \cite{ashraphijuo3}.

\begin{lemma}\label{thmten1TT}
For almost every $\mathcal{U}$, there are only finitely many completions of rank ${\underline{ r}^*}$ of the sampled tensor if and only if Assumptions $\mathcal{A}_{\underline{ r}^*}^{\text{TT}}$ and $\mathcal{B}_{\underline{ r}^*}^{\text{TT}}$ hold.
\end{lemma}

\begin{definition}
Let $\mathcal{S}_{{\Omega}}$ denote the set of all rank vectors $\underline{ r}$ such that both Assumptions $\mathcal{A}_{\underline{ r}}^{\text{TT}}$ and $\mathcal{B}_{\underline{ r}}^{\text{TT}}$ hold.
\end{definition}


The following lemma will be used in Lemma \ref{tensorTTminimorder}.

\begin{lemma}\label{ineqcosTT}
$u_i \leq \min\{u_{i-1}n_{i},u_{i+1}n_{i+1}\}$ for $1 \leq i \leq d-1$.
\end{lemma}

\begin{proof}
We first show that $u_i \leq u_{i-1}n_{i}$, which is easily verified for $i=1$ as $\mathbf{\widetilde U}_1$ includes $n_1$ rows and $u_0 = 1$, and therefore assume that $ i > 1$. Define the $(d-1)$-way tensor $\mathcal{U}^{l_i} \in \mathbb{R}^{n_1 \times  \dots \times n_{i-1}  \times n_{i+1} \times \dots \times n_d}$ such that $ \mathcal{U}^{l_i} (x_1,\dots,x_{i-1},x_{i+1}, \dots, x_d) = \mathcal{U} (x_1,\dots,x_{i-1},l_i,x_{i+1}, \dots, x_d) $ for $1 \leq i \leq d$ and $1 \leq l_i \leq n_i$. Also, recall that $\mathbf{\widetilde U}_{(i-1)}^{l_i}$ denotes the $(i-1)$-th unfolding of $ \mathcal{U}^{l_i}$. Observe that $\mathbf{\widetilde U}_{(i-1)}^{l_i}$ is a subset of columns of matrix $\mathbf{\widetilde U}_{(i-1)}$ (those columns that correspond to the entries of $\mathcal{U}$ with the $i$-th component of the location equal to $l_i$). Therefore, $\text{rank}(\mathbf{\widetilde U}_{(i-1)}^{l_i}) \leq \text{rank}(\mathbf{\widetilde U}_{(i-1)}) = u_{i-1}$. 

On the other hand, observe that $\mathbf{\widetilde U}_{(i-1)}^{l_i}$ is a subset of rows of $\mathbf{\widetilde U}_{(i)}$ (those rows that correspond to the entries of $\mathcal{U}$ with the $i$-th component of the location equal to $l_i$). Hence, the union of rows of $\mathbf{\widetilde U}_{(i-1)}^{l_i}$'s for $1 \leq l_i \leq n_i$ constitute all rows of $\mathbf{\widetilde U}_{(i)}$. Therefore, $u_i = \text{rank}(\mathbf{\widetilde U}_{(i)}) \leq $ $\sum_{l_i=1}^{n_i} \text{rank}(\mathbf{\widetilde U}_{(i-1)}^{l_i}) $ $ \leq n_i u_{i-1}$. Similarly, we can show that $u_i \leq u_{i+1}n_{i+1}$ to complete the proof.
\end{proof}


\begin{lemma}\label{tensorTTminimorder}
Assume $\underline{ r} \in  \mathcal{S}_{{\Omega}} $. Then, for any $\underline{ r}^{\prime} \preceq  \underline{ r}$, we have $\underline{ r}^{\prime} \in  \mathcal{S}_{{\Omega}} $.
\end{lemma}

\begin{proof}
Note that the dimension of the manifold corresponding to $\underline{ r}$ is $\sum_{i=1}^{d} u_{i-1}n_iu_i - \sum_{i=1}^{d-1} u_i^2$ \cite{ashraphijuo3}. If we reduce the value of $u_i$ by one, the value of the mentioned dimension reduces by $u_{i-1}n_i + u_{i+1}n_{i+1} - 2u_i + 1$. According to Lemma \ref{ineqcosTT}, $u_{i-1}n_i + u_{i+1}n_{i+1} - 2u_i + 1$ is greater than zero, and therefore $\underline{ r}^{\prime} \preceq  \underline{ r}$ results that the dimension of the manifold corresponding to $\underline{ r}$ is greater than that corresponding to $\underline{r}^{\prime}$. The rest of the proof is similar to the proof of Lemma \ref{sinvorderkam}.
\end{proof}

\begin{definition}
Define $\hat {\mathcal{S}}_{\Omega} $ as the set of all rank vectors $\underline{ r}  \in \mathcal{S}_{{\Omega}}$ such that there exists a rank vector $\underline{ r}^{\prime} \in \mathcal{S}_{{\Omega}}$ with $\underline{ r} \preceq \underline{ r}^{\prime}$ and $u_{d-1} < u_{d-1}^{\prime}$ (instead of $u_{d-1} \leq u_{d-1}^{\prime}$). Note that $\hat {\mathcal{S}}_{\Omega} $ also satisfies the property in Lemma \ref{tensorTTminimorder}.
\end{definition}

\begin{theorem}\label{thmmat2tensorTT}
For almost every $\mathcal{U}$, with probability one, exactly one of the following statements holds

(i) $\underline{ r}^* \in \hat{\mathcal{S}}_{{\Omega}}${\rm ;}

(ii) For any arbitrary completion of the sampled tensor $\mathcal{U}$ of rank  $\underline{ r}$, we have $\underline{ r} \notin \hat{\mathcal{S}}_{{\Omega}}$.
\end{theorem}

\begin{proof}
Similar to the proof of Theorem \ref{thmmat2}, to complete the proof it suffices to show that the assumption $\underline{ r}^* \notin \hat{\mathcal{S}}_{{\Omega}}$ results that there exists a completion of $\mathcal{U}$ of rank $\underline{ r}$, where  $\underline{ r} \in \hat{\mathcal{S}}_{{\Omega}}$, with probability zero. Define the multiplication $\mathcal{U}^{(1)} \dots \mathcal{U}^{(d-1)}$ in \eqref{TTeq1} as the basis of the rank $\underline{ r}$ TT decomposition of $\mathcal{U}$. Then, by considering the $(d-1)$-th unfolding of $\mathcal{U}^{(1)} \dots \mathcal{U}^{(d-1)}$ in TT decomposition we obtain a matrix factorization of the $(d-1)$-th unfolding of $\mathcal{U}$. The rest of the proof is similar to the proof of Theorem \ref{thmmat2}.
\end{proof}

Similar to Theorem \ref{thmmat2tensorTT}, we can show the following.

\begin{corollary}\label{colthmtensorTT1}
Consider a subset $\hat{\mathcal{S}}_{{\Omega}}^{\prime}$ of $\hat{\mathcal{S}}_{{\Omega}}$ such that for any two members of $\hat{\mathcal{S}}_{{\Omega}}$ that $\underline{ r}^{\prime \prime} \preceq \underline{ r}^{\prime}$ and $\underline{ r}^{\prime} \in \hat{\mathcal{S}}_{{\Omega}}^{\prime}$ we have $\underline{ r}^{\prime \prime} \in \hat{\mathcal{S}}_{{\Omega}}^{\prime}$. Then, with probability one, exactly one of the followings holds

(i) $\underline{ r}^* \in \hat{\mathcal{S}}_{{\Omega}}^{\prime} ${\rm ;}

(ii) For any arbitrary completion of $\mathcal{U}$ of rank vector $\underline{ r}$, we have $\underline{ r} \notin \hat{\mathcal{S}}_{{\Omega}}^{\prime} $.
\end{corollary}

\begin{corollary}\label{colthmtensorTT1nw}
Assuming that there exists a completion of $\mathcal{U}$ with rank vector $\underline{ r}$ such that $\underline{ r} \in \hat {\mathcal{S}}_{\Omega}$, we conclude that with probability one $\underline{ r}^* \preceq \underline{ r} $.
\end{corollary}

The following lemma is Lemma $14$ in \cite{ashraphijuo3}, which ensures that Assumptions $\mathcal{A}_{\underline{ r}}^{\text{TT}}$ and $\mathcal{B}_{\underline{ r}}^{\text{TT}}$ hold with high probability.

\begin{lemma}\label{prsTTfinRD}
Define $m= \sum_{k=1}^{d-2} u_{k-1}u_{k}$, $M = n \sum_{k=1}^{d-2} u_{k-1}u_k -\sum_{k=1}^{d-2} u_k^2$ and $u^{\prime}= \max \left\{  \frac{u_1}{u_0}  , \dots,  \frac{u_{d-2}}{u_{d-3}} \right\}$. Assume that $n_1=n_2=\dots = n_d=n$, $n > \max\{m,200\} $ and $u^{\prime} \leq \min\{\frac{n}{6},  u_{d-2}\}$ hold. Moreover, assume that the sampling probability satisfies
\begin{eqnarray}\label{prsTTbdfinRD}
p > \frac{1}{n^{d-2}} \max\left\{27 \ \log \left( \frac{n}{\epsilon} \right) + 9 \ \log \left( \frac{2M}{\epsilon} \right) + 18, 6u_{d-2}\right\} + \frac{1}{\sqrt[4]{n^{d-2}}}.
\end{eqnarray}
Then, with probability at least $ (1- \epsilon) \left( 1-\exp(-\frac{\sqrt{n^{d-2}}}{2}) \right)^{n^2}$, we have $\underline{ r} \in \mathcal{S}_{\Omega}$.
\end{lemma}

The following corollary is the probabilistic version of Corollary \ref{colthmtensorTT1nw}.

\begin{corollary}\label{colthmtensorTT1nwtq}
Assuming that there exists a completion of the sampled tensor $\mathcal{U}$ of TT rank $\underline{ r}$ such that the assumptions in Lemma \ref{prsTTfinRD} hold and the sampling probability satisfies \eqref{prsTTbdfinRD}, then with probability at least $ (1- \epsilon) \left( 1-\exp(-\frac{\sqrt{n^{d-2}}}{2}) \right)^{n^2}$ we have $\underline{ r}^* \preceq \underline{ r}$.
\end{corollary}

\section{Conclusions}\label{conclusection}

We make use of the recently developed algebraic geometry analyses that study the fundamental conditions on the sampling patterns for finite completability under a number of low-rank matrix and tensor models to treat the problem of rank approximation for a partially sampled data. Particularly, the goal is to approximate the unknown scalar or vector rank based on the sampling pattern and the rank of a given completion. A number of data models have been treated, including single-view matrix, multi-view matrix, CP tensor, tensor-train tensor and Tucker tensor. First we have provided an upper bound on the unknown scalar rank (for single-view matrix and CP tensor) and an component-wise upper bound on the vector rank (for multi-view matrix, Tucker tensor and TT tensor) with probability one assuming that the sampling pattern satisfies the proposed combinatorial conditions. Moreover, we have also provided probabilistic versions of such bounds that hold with high probability assuming that the sampling probability is above a threshold. In addition, for single-view matrix and CP tensor, these upper bounds can be exactly equal to the unknown scalar rank given the lowest-rank completion. To illustrate how tight our proposed upper bounds are, we have provided some numerical results for the single-view matrix case in which we applied the nuclear norm minimization to find a low-rank completion of the sampled data and observe that the proposed upper bound is almost equal to the true unknown rank.

\bibliographystyle{IEEETran}
\bibliography{bib}

\end{document}